\documentclass{article}

\usepackage{microtype}
\usepackage{graphicx}
\usepackage{subfigure}
\usepackage{booktabs} 

\usepackage[OT1]{fontenc}
\usepackage{amsmath,amssymb,amsthm}
\usepackage{tabularx}
\usepackage{hyperref}
\usepackage{cleveref}
\usepackage{color}
\usepackage{bm}
\usepackage[dvipsnames]{xcolor}
\usepackage{tikz}
\usepackage{tikz-cd,mathtools}
\usepackage{mathtools}
\usepackage[normalem]{ulem}
\usepackage{xfrac}
\usepackage{stackengine}

\usetikzlibrary{arrows, matrix} 
\tikzcdset{
    every label/.append style = {scale=0.8,yshift=0.5ex},
    cells={nodes={font=\normalsize}}
    } 

\graphicspath{{./pic/}}

\newtheorem{definition}{Definition}
\newtheorem{theorem}{Theorem}
\newtheorem{lemma}{Lemma}
\newtheorem{remark}{Remark}
\newtheorem{example}{Example}

\newcommand{\GFNN}{{GFNN}}
\newcommand{\qp}{\bm{p}, \bm{q}}
\newcommand{\qpt}[1]{\bm{p}_{#1}, \bm{q}_{#1}}

\newcommand{\Iphi}{\bm{I}, \bm{\varphi}}
\newcommand{\Iphit}[1]{\bm{I}_{#1}, \bm{\varphi}_{#1}}
\newcommand{\Jtheta}{\bm{J}, \bm{\theta}}
\newcommand{\Jthetat}[1]{\bm{J}_{#1}, \bm{\theta}_{#1}}
\newcommand{\dataset}{$\left\{ {\left({\left[ \qpt{i, j} \right]}_{i=0}^{N_j}\right)}_{j=1}^M  \right\}$}

\newcommand{\R}{\mathbb{R}}
\newcommand{\Z}{\mathbb{Z}}
\newcommand{\N}{\mathbb{N}}
\newcommand{\T}{\mathbb{T}}
\newcommand{\C}{\mathbb{C}}

\newcommand{\Dc}{\mathcal{D}}
\newcommand{\Oc}{\mathcal{O}}
\newcommand{\Tc}{\mathcal{T}}
\newcommand{\Bc}{\mathcal{B}}

\newcommand{\Lc}{\mathcal{L}}
\newcommand{\norm}[1]{\left\|#1\right\|}

\newcommand{\abs}[1]{\left|#1\right|}
\newcommand{\Mod}[1]{\ \mathrm{mod}\ #1}
\newcommand{\union}{\bigcup}

\crefname{algorithm}{Algorithm.}{Algorithms.}
\crefname{equation}{Eq.}{Eqs.}
\crefname{theorem}{Thm.}{Thms.}
\crefname{lemma}{Lemma.}{Lemmas.}
\crefname{remark}{Rmk.}{Rmks.}
\crefname{figure}{Fig.}{Figs.}
\crefname{proof}{Proof.}{Proofs.}

\usepackage[accepted]{icml2021}

\icmltitlerunning{Data-driven Prediction of General Symplectic Dynamics via Generating Function Neural Network (GFNN)}

\begin{document}

\twocolumn[
    \icmltitle{Data-driven Prediction of General Hamiltonian Dynamics \\via Learning Exactly-Symplectic Maps}




\begin{icmlauthorlist}
\icmlauthor{Renyi Chen}{gt-math}
\icmlauthor{Molei Tao}{gt-math}
\end{icmlauthorlist}

\icmlaffiliation{gt-math}{School of Mathematics, Georgia Institute of Technology, Atlanta, GA, USA}

\icmlcorrespondingauthor{Molei Tao}{mtao@gatech.edu}

\icmlkeywords{learning dynamics, time series, mechanical systems, global error, integrable and chaotic dynamics}

\vskip 0.3in
]



\printAffiliationsAndNotice{}  

\begin{abstract}

We consider the learning and prediction of nonlinear time series generated by a latent symplectic map. A special case is (not necessarily separable) Hamiltonian systems, whose solution flows give such symplectic maps. For this special case, both generic approaches based on learning the vector field of the latent ODE and specialized approaches based on learning the Hamiltonian that generates the vector field exist. Our method, however, is different as it does not rely on the vector field nor assume its existence; instead, it directly learns the symplectic evolution map in discrete time. Moreover, we do so by representing the symplectic map via a generating function, which we approximate by a neural network (hence the name GFNN). This way, our approximation of the evolution map is always \emph{exactly} symplectic. This additional geometric structure allows the local prediction error at each step to accumulate in a controlled fashion, and we will prove, under reasonable assumptions, that the global prediction error grows at most \emph{linearly} with long prediction time, which significantly improves an otherwise exponential growth. In addition, as a map-based and thus purely data-driven method, GFNN avoids two additional sources of inaccuracies common in vector-field based approaches, namely the error in approximating the vector field by finite difference of the data, and the error in numerical integration of the vector field for making predictions. Numerical experiments further demonstrate our claims. 
\end{abstract}

\section{Introduction}
Given a collection of sequences, each being a multidimensional time series produced by the same latent mechanism, we consider learning this mechanism and predicting a sequence's future evolution. More precisely, suppose there is an unknown (possibly highly nonlinear) map $\phi$ that evolves any initial condition in discrete time $i$ according to
\[ 
    x_{i+1} = \phi(x_i).
\]
Provided with training data $\{x_{i,j}\}$, where each $j \in [M]$ corresponds to such a sequence ($i \in \{0,\ldots,N_j\}$), we'd like to learn $\phi$ purely from the data, or more precisely an approximation of its image $\tilde{\phi}(x) \approx \phi(x)$ for any $x$ in the problem domain. This way, one can for example perform continuation of existing sequences via $\tilde{x}_{i+1,j} = \tilde{\phi}(\tilde{x}_{i,j})$ for $i \geq N_j$ and $\tilde{x}_{N_j,j}=x_{N_j,j}$, or predict a sequence evolved from a new initial condition via $y_{i+1} = \tilde{\phi} (y_i)$ for $i>0$.

This task of course appears in many contexts (see e.g., Sec.\ref{sec:relatedWorks}). This article considers a very specific one, for which the latent map $\phi$ is assumed to be \emph{symplectic}: for simplicity we will work with finite dim. vector spaces equipped with canonical symplectic structure, which means each $x$ can be written as $x=[\qp]$ where $\qp \in \mathbb{R}^d$, and the Jacobian of $\phi$ satisfies
\[
    (\phi')^T J \phi' = J,
\]
where $J=\begin{bmatrix} 0 & I \\ -I & 0 \end{bmatrix}$ is a $2d$-by-$2d$ matrix with $0$ and $I$ being $d$-by-$d$ blocks.

The consideration of symplectic evolution maps is largely motivated by the learning and prediction of mechanical behaviors, which recently attracted significant attention  (see Sec.\ref{sec:relatedWorks} 3rd paragraph). More precisely, if the latent evolution mechanism is provided by a Hamiltonian mechanical system, each time series is given by a solution to the Hamiltonian ODE system sampled at discretized time points. That is, $x_i=[\bm p_i,\bm q_i]$, $\bm p_i = \bm p(ih), \bm q_i = \bm q(ih)$, with
\begin{equation}
    \dot{\bm p}(t) = -\frac{\partial H}{\partial \bm q} (\bm p(t),\bm q(t)),\,
    \dot{\bm q}(t) = \frac{\partial H}{\partial \bm p} (\bm p(t),\bm q(t)), \,
    \label{eq:HamiltonianODE}
\end{equation}
where $H(\cdot,\cdot)$ is the latent Hamiltonian function and $h>0$ is the sampling time step. In this case, the latent $\phi$ we're trying to learn is the flow map of ODE \eqref{eq:HamiltonianODE}, defined as
\[
    \phi [\bm p(t), \bm q(t)] := [\bm p(t+h), \bm q(t+h)],    \quad \forall t.
\]
Given any $h>0$, the corresponding $\phi$ is a symplectic map (e.g., \citealp{goldstein1980classical}), and that is why the learning of symplectic maps is relevant.

Worth noting is, a popular and successful line of thoughts is based on learning the right hand side of the latent ODE, which in our case (eq.\ref{eq:HamiltonianODE}) amounts to either generic approaches that learn the vector-valued function $f(\qp) := [-\partial H/\partial \bm q, \partial H/\partial \bm p]$ (see Sec.\ref{sec:relatedWorks} 1st paragraph), or specialized methods that utilize the problem structure and learn the scalar-valued function $H(\qp)$ (see Sec.\ref{sec:relatedWorks} 3rd paragraph). This article, however, is based on a different idea, namely to directly \textbf{learn the evolution map} $\phi$.

The advantages of doing so include: (i) \underline{Generality}: it works no matter whether there is an underlying ODE system (see Sec.\ref{sec:Results:standardMap} for an example where there isn't). (ii) \underline{Local Accuracy}: for purely data driven problems, learning the map has \emph{one} source of error, namely the approximation error of the map, whereas learning vector field generally has \emph{three}: first one has to estimate the vector field from data, for example by finite difference which incurs error, then there is approximation error of the vector field, and finally the vector field needs to be numerically integrated in order to make predictions and this creates error too. (iii) \underline{Symplecticity (and Global Accuracy)}: we will propose a simple way to exactly maintain the symplecticity of $\phi$ despite using an approximation\footnote{Vector-field-based methods can also be designed to make symplectic predictions; see Sec.\ref{sec:relatedWorks} 4th paragraph.}; we will also rigorously show, when considering how local error accumulates for making long time predictions, exact symplecticity can help significantly by keeping the error accumulation additive, so that global error grows linearly instead of exponentially.

\begin{example} 
    To make things concrete, consider latent Hamiltonian dynamics $\dot{x}=Ax$ where $A=-A^T$. A vector-field-based method aims at learning the function $Ax$, but if discrete time-series are the only available data, it actually learns $\tilde{A}x$ instead, where  $\tilde{A}=(\exp(Ah)-I)/h$ if 1st-order finite difference is used for estimating the vector-field. A map-based method seeks the map $\exp(A h)x$ instead. When later making predictions, the vector-field-based method numerically integrates the vector field $\tilde{A}x$, which oftentimes corresponds to constructing a polynomial in $h$ approximation of $\exp(\tilde{A} h)$; on the contrary, a map-based method requires no numerical integration. A similar comparison holds for nonlinear cases too.
\end{example}

In order to learn a \emph{symplectic} evolution map, we use a tool known as generating functions, which have one-to-one correspondence with symplectic maps. We use a Neural Network, however not for approximating the latent symplectic map, but to approximate its corresponding Generating Function (the method is thus called GFNN). By doing so, the associated evolution map is always symplectic, whether or not it is a good approximation of the latent evolution map, and an appropriate neural network, even just a feedforward one, will be a good approximation after training (see Rmk.\ref{rmk:universalApprox}). This symplectic map representation is intrinsic, purely due to the symplectic structure, and no regularization is used.

Moreover, the guaranteed symplecticity originated from  the generating function technique allows us to obtain a nontrivial, linearly growing bound on the prediction error:

\begin{theorem}[Informal version of Thm.\ref{thm:main}]
Consider latent dynamics far from chaos (more precisely, being integrable). If the latent generating function is approximated with $\leq \varepsilon$ error in first derivatives, then except for a set of initial conditions whose measure goes to 0 as $\varepsilon \downarrow 0$, the deviation between the predicted sequence $\left( \qpt{0} \right), \left( \qpt{1}, \right), \ldots$ and the true sequence satisfies
    \begin{align*}
    \left\{
    \begin{aligned}
    \norm{\bm p_n - \bm p\left( n h \right)}_2
         \le C \cdot \left( n h \right) \cdot \varepsilon, \\
    \norm{\bm q_n - \bm q\left( n h \right)}_2
         \le C \cdot \left( n h \right) \cdot \varepsilon, \\
    \end{aligned}
    \right.
    \quad \forall n \le h^{-1} \varepsilon^{-1},
    \end{align*}
    for some constant $C>0$, where $n$ is the number of prediction steps and $h$ is the sampling time step of the data.
\end{theorem}
The merit of this bound lies in long time predictions: note $n$ can be arbitrarily large as $\varepsilon$ can be infinitesimal ($h$ is fixed by the training data, and $nh$ is the physical prediction time).

\textbf{A brief summary of main contributions:}

$\bullet$ (Algorithm) \emph{Learn map instead of vector field. Exact symplecticity guaranteed by generating function representation.}

$\bullet$ (Theory) \emph{Linear bound on long-time prediction error.}

$\bullet$ (Validation) \emph{Systematic empirical investigations.}

%

\subsection{Related Works and Discussion}
\label{sec:relatedWorks}

Learning and then predicting dynamics from data is an extremely active research direction. While it is impossible to review all important works, we first mention the classical area of time series (e.g., \citealp{box2015time, abarbanel2012analysis, kantz2004nonlinear, bradley2015nonlinear}), where latent differential equations may or may not be assumed. For cases where a latent differential equation is believed to exist, which may correspond to a complex and/or un-modeled underlying dynamical process, some seminal works include \cite{baake1992fitting, bongard2007automated, schmidt2009distilling},
and more recent progress include
those based on learning (part of) the vector field via sparse regression of a library (e.g., \citealp{brunton2016discovering, tran2017exact, schaeffer2018extracting, lu2019nonparametric, rudy2017data, kang2019ident, reinbold2021robust}),
learning the vector field via neural network (e.g., \citealp{raissi2018multistep, rudy2019deep, qin2019data, long2018pde}),
and learning the vector field via other approaches such as Gaussian processes (e.g., \citealp{raissi2018hidden}).

`Model-free' approaches that are based on machine learning techniques for sequences have also been proposed, such as
\cite{bailer1998recurrent} (vanilla RNN),
\cite{wang2017new} (LSTM),
\cite{pathak2018model} (reservoir computing), 
\cite{mukhopadhyay2020learning} (CNN), and
\cite{shalova2020tensorized} (transformer).

Faced with the extreme success of these generic methods, interests have also been growing in incorporating domain knowledge and specific structures of the underlying problems into the otherwise-black-box schemes (see e.g., \citealp{raissi2019physics}).
In terms of mechanical problems modeled by Hamiltonian systems, seminal progress include HNN \cite{greydanus2019hamiltonian} and an independent work \cite{bertalan2019learning}, SRNN \cite{chen2019symplectic}, SympNets \cite{jin2020sympnets}, and \cite{lutter2018deep,toth2019hamiltonian,
zhong2019symplectic,
wu2020structure,
xiong2021nonseparable}, all of which, except SympNets, are related to learning some quantity that produces the Hamiltonian vector field.

In particular, both HNN and SRNN are based on the great idea of learning (using a neural network) the Hamiltonian that generates the vector field (VF), instead of learning the VF itself; this improves accuracy as the Hamiltonian structure of the VF will not be lost due to approximation. 
HNN learns the Hamiltonian by matching its induced VF with the latent VF (when such information is unavailable, for example in a purely data driven context, data-based approximation such as finite-difference is needed). Then it predicts by numerically integrating the learned VF, and for this we note a Hamiltonian VF doesn't guarantee the symplecticity of its integration\footnote{Unless a symplectic integrator is used. Note the seminal work of HNN used RK45 which is not symplectic, however with small error tolerance (thus good precision but high computation cost).}.
SRNN, on the other hand, learns the Hamiltonian by matching its symplectic integration with the training sequences, and its prediction is then given by symplectic integration of the learned Hamiltonian. It is therefore the closer to GFNN as it essentially learns a symplectic map; it is just that SRNN represents this map by a symplectic integration of a neural-network-approximated Hamiltonian, whereas GFNN represents it by a neural-network-approximated generating function. Because of this, SRNN doesn't need finite-difference approximation and has good prediction accuracy, but it only works for symplectic maps originated from Hamiltonian ODEs, and its accuracy is hampered if the latent Hamiltonian is nonseparable\footnote{The original SRNN is based on symplectic integrators for separable Hamiltonians, and nonsymplectic integrators for nonseparable ones; see Footnote \ref{ftnt:SRNN_nonseparable} for additional information.}.



In comparison, GFNN is not based on Hamiltonian vector fields. It is purely data driven, always symplectic, and works the same for separable-, nonseparable-, or even non-Hamiltonian latent systems.


Worth mentioning is the clever recent work of SympNets \cite{jin2020sympnets}, which also enjoys most of the aforementioned qualitative features of \GFNN{}. It complements \GFNN{} and echoes with our view that directly approximating symplectic maps (instead of Hamiltonian vector fields) in an exactly symplectic way is advantageous (note SRNN can also be seen as a (different) way of doing so). \emph{Algorithmically}, SympNets stack up triangular maps (inspired by symplectic integrator) to construct specialized (new) neural networks, which represent only symplectic maps, and then use them to directly approximate the latent evolution map; GFNN on the other hand uses generating function to indirectly represent the evolution map, and because of its mathematical structure, exact symplecticity is automatically guaranteed, and no special neural network is needed for representing the generating function. Consequently, the \emph{theory} of SympNets is devoted to a universal approximation theorem that characterizes the local prediction error, whereas we focus on the global prediction error (i.e., error after many steps of prediction, instead of one) and rigorously show a nontrivial fact that local errors only accumulate linearly into global error; no approximation theorem needed as it's already established for generic networks. In terms of \emph{performance}, we observe SympNets outperforming vector-field-based approaches (as reasoned above), but GFNN has further improved performance; see e.g., \cref{fig:2body}, for which we tried up to 30 layers with 10 sublayers using SympNets' code (both LA- and G-SympNets) and plotted its best result, namely LA-SympNets with 30 layers and 10 sublayers (c.f., here \GFNN{} used 5 layers). We feel SympNets generally require a significantly deeper network than \GFNN{} to achieve high approximation power, but then training and computational challenges may arise.

One more remark is, a good amount of existing work considered predicting chaotic dynamics, but a major part of this work is concerned with structured, non-chaotic dynamics, for which controlled long time (strong) accuracy in individual trajectories becomes possible. Predicting chaos is nontrivial, but sometimes the existence of a chaotic attractor makes the system forgiving, and because of that, prediction errors do not accumulate as much as they can in non-chaotic systems. Besides, one often cares more about statistical accuracy for chaotic systems (e.g., \citealp{tsai2020learning}), as opposed to strong accuracy in trajectory (which usually grows too fast in chaos; see Rmk.\ref{rmk:chaosPositiveLyapunov}). Meanwhile, accurately predicting the trajectory of non-chaotic systems is desirable in numerous applications. Nevertheless, \GFNN's predictive power for chaos will be empirically confirmed too.

\section{Methods}

\subsection{Symplectic Map and Generating Function}
As we do not assume or seek a latent ODE system but directly approximate the evolution map, a representation of this symplectic map is essential. Instead of directly approximating it, which has the extra difficulty of losing symplecticity (which has to be exact), we use a mathematical tool known as generating function. Let us be more specific:

Firstly, given a (type-2) generating function, there is an associated symplectic map (a.k.a. canonical transformation):

\begin{lemma}
    Consider a differentiable function $F(\bm q, \bm P)$ which shall be called a generating function. The map $[\bm p, \bm q] \mapsto [\bm P,\bm Q]$ implicitly defined by
    $
        \bm p = \frac{\partial F}{\partial \bm q}(\bm q,\bm P)$,
    $
        \bm Q = \frac{\partial F}{\partial \bm P}(\bm q,\bm P),
    $
    is a symplectic map.
\end{lemma}

\begin{proof}
    See e.g., \cite{goldstein1980classical}.
\end{proof}

The converse is also true, as long as the latent map doesn't correspond to an evolution time too long (otherwise singularities can be developed):

\begin{lemma}
    For any infinitesimal symplectomorphism (i.e., symplectic map) on $T^*\mathbb{R}^d$ (i.e., vector phase space), there is a corresponding generating function.
\end{lemma}
\begin{proof}
    This is because the first cohomology group of $T^*\mathbb{R}^d$ is trivial; see e.g., \cite{da2001lectures}.
\end{proof}

\begin{remark}[generating functions and Hamiltonian system]
    We do not assume the latent map that generates the data in discrete time corresponds to an underlying Hamiltonian ODE system in continuous time. There are symplectic maps that do not have such correspondence (see e.g., Sec.\ref{sec:Results:standardMap}). 
    
    On the other hand, given a Hamiltonian system, its flow map, defined as $\phi^t : [\bm p(0), \bm q(0)] \mapsto [\bm p(t), \bm q(t)]$, is symplectic for any $t$. Therefore, there is a family of corresponding generating functions $F(\bm q, \bm P, t)$, each of which generates the symplectic map $[\bm P, \bm Q]=\phi^t [\bm p, \bm q]$. Moreover, the relation between the Hamiltonian $H$ and $F$ can be made more direct via the Hamilton-Jacobi PDE:
    $
        H\left(\frac{\partial F}{\partial \bm q}, \bm q, t\right) + \frac{\partial F}{\partial t} = 0.
    $
\end{remark}

Because of their 1-to-1 correspondence, instead of approximating the symplectic evolution map $\phi:[\bm p,\bm q]\mapsto[\bm P,\bm Q]$, we use a Feedforward Neural Network to approximate the corresponding generating function $F(\bm q,\bm P)$. This way, no matter how much error the FNN has in approximating $F$, it always gives to an evolution map that is \emph{exactly} symplectic.

\subsection{Learning Based on Generating Function Training}
The type-2 generating function corresponding to a $h$-time flow map is $F\left( \bm q, \bm P \right)=\bm q \cdot \bm P + \mathcal{O}(h)$, and what varies across different problems is inside the $\mathcal{O}(h)$ term. Therefore, for easier training we learn an equivalent, modified generating function $S_h$, defined through $F \left( \bm q, \bm P \right) = \bm q \cdot \bm P + h \cdot S_h \left( \bm q, \bm P \right)$. It generates a sequence via iteration
\begin{align}\label{eq:Sh_seq}
    \left\{
    \begin{aligned}
        \bm p_{i} = \bm p_{i+1} + h \cdot \partial_1 S_h \left( \bm q_{i}, \bm p_{i+1} \right), \\
        \bm q_{i+1} = \bm q_{i} + h \cdot \partial_2 S_h \left( \bm q_{i}, \bm p_{i+1} \right), \\
    \end{aligned}
\right.
\end{align}
as long as an initial condition $[\qpt{0}]$ is provided.

To learn the latent $S_h$, GFNN uses a neural-network approximation $S_h^\theta$, and trains for a good parameterization $\theta$ to best satisfy \eqref{eq:Sh_seq}. See \cref{algorithm:1}.
\begin{algorithm}[ht]
 \caption{\GFNN{}}\label{algorithm:1}
\begin{algorithmic}
\STATE {\bfseries{Data:}}
The data set \dataset is observed from sequences generated by a symplectic map $\phi^h$, with $\left[ \qpt{i, j} \right] \in \Dc \subseteq \R^d \times \R^d \cong T^* \R^d$ and $\qpt{i+1, j} = \phi^h \left( \qpt{i, j} \right)$.
\STATE {\bfseries{Training:}}
Optimize the loss function
\begin{align}\label{eq:loss}
    \begin{split}
        & \Lc_{GFNN} = \frac{1}{\sum_{j=1}^M N_j} \sum_{j=1}^M \sum_{i=0}^{N_j-1}\\
        & \quad \Big( \norm{h \partial_2 S_h^\theta(\bm q_{i, j}, \bm p_{i+1, j})-(\bm q_{i+1, j} - \bm q_{i, j})}_2^2 \\
    & \quad +\norm{h \partial_1 {S}_h^\theta(\bm q_{i, j}, \bm p_{i+1, j})-(\bm p_{i, j} - \bm p_{i+1, j})}_2^2 \Big).
    \end{split}
\end{align}
with respect to neural network parameters $\theta$ (see Appendix for our experimental details).
\STATE {\bfseries{Prediction:}}
Given any initial condition $(\bm q_0, \bm p_0) \in \Dc$,
one step evolution to $(\tilde{\bm q}_1, \tilde{\bm p}_1)$ can be solved from
\begin{align}\label{eq:prediction}
    \left\{
    \begin{aligned}
        \bm p_0 = \tilde{\bm p}_1 + h \cdot \partial_1 S_h^\theta \left( \bm q_0, \tilde{\bm p}_1 \right), \\
        \tilde{\bm q}_1 = \bm q_0 + h \cdot \partial_2 S_h^\theta \left( \bm q_0, \tilde{\bm p}_1 \right). \\
    \end{aligned}
\right.
\end{align}
This can be iterated.
\end{algorithmic}
\end{algorithm}

\section{Global Error Analysis}\label{sec:error_analysis}

We now show that, under reasonable assumptions, \GFNN{}'s prediction will be close to the true sequence (continued by the latent $\phi$) for a very long time, as a linearly growing long time error bound will be established. This will be contrasted with an obtainable exponentially growing error bound for generic vector-field-based methods. The latter methods are of course more versatile but they do not utilize the special symplectic structure. Proofs are based on normal form and KAM-type techniques and deferred to Appendix.

The main condition needed for this mild error growth is integrability, which, very roughly speaking, requires the latent system to be far from chaos, but it could still be highly nonlinear; see e.g., \cite{arnol2013mathematical}. In order to make it precise, some mathematical preparations are needed, but one can jump to Thm.\ref{thm:main} for the main results if preferred.

\begin{definition} A function $g(\qp)$ is called a 1st-integral or a constant of motion of the dynamics if it remains constant as $\qp$ evolves in (continuous or discrete) time.
\end{definition}

\begin{definition}
    The (canonical) Poisson bracket of two arbitrary functions $f(\qp), g(\qp)$ is another function defined as $\{f,g\}:=\langle \partial f/\partial q, \partial g/\partial p \rangle - \langle \partial f/\partial p, \partial g/\partial q \rangle$.
\end{definition}

\begin{theorem}[Arnold-Liouville]
Consider a $d$-degree-of-freedom Hamiltonian system. Assume there exist d independent $1$st integrals in the sense that the Poisson bracket of every pair is $0$. If the $d$-dim. surfaces implicitly defined by the level sets of those $1$st integrals are compact, then there exists a canonical transformation from $\qp$ to $\bm I, \bm \varphi$, such that $\bm \varphi$ can be defined on the $d$-torus, and in the new variables the Hamiltonian only depends on $\bm I$. In this case, $\bm I$, $\bm \varphi$, and the Hamiltonian are respectively called the action, angle variables, and an integrable Hamiltonian.
\end{theorem}
\begin{proof} See e.g., \cite{arnol2013mathematical}.
\end{proof}

\begin{remark}
In an integrable system, the action variables $\bm I$ remain constants (they are canonical versions of the first integrals), while the angle variables $\bm \varphi$ evolve on an invariant torus $\left\{ \bm I=\bm I\left(0\right), \bm \varphi \in \T^d \right\}$, where 
\begin{align*}
    \T^d & =\R^d\slash (2\pi\Z^d)
    =\left\{ \left( \varphi_1, \ldots, \varphi_d \right) \Mod 2\pi; \varphi_i \in \R \right\}.
\end{align*}
It is easy to show that the fixed time step generating function for an integrable system takes the form of $S_h \left( \bm I_0, \bm \varphi_1 \right) = H\left( \bm I_0 \right)$, as the exact time-h flow is defined as the following
\begin{align}\label{eq:Iphi_dynamics}
    \left\{
    \begin{aligned}
        \bm I_{1} & = \bm I_{0}, \\
        \bm \varphi_{1} & = \bm \varphi_{0} + h \, \partial_{1} S_h \left( \bm I_{0}, \bm \varphi_1 \right) = \bm \varphi_{0} + h \, \nabla H \left( \bm I_0 \right). \\
    \end{aligned}
    \right.
\end{align}
Denote $\nabla H \left( \bm I \right)$ by $\bm \omega \left( \bm I \right) = \left[ \omega_1 \left( \bm I \right), \ldots, \omega_d \left( \bm I \right) \right]$.
It can be directly seen from \cref{eq:Iphi_dynamics} that $\omega_i \left( \bm I \right)$ represents the change rate (i.e., frequency) of the angle variable $\varphi_i$.
\label{rmk_invariantTori}
\end{remark}

Our theory works for almost all initial frequencies, and to describe what are the exceptions we need the following definition, which generalizes irrational numbers in some sense.

\begin{definition}[$\left(\gamma, \nu\right)$-Diophantine condition\footnote{also known as strong non-resonance condition}]
    Frequency vector $\bm \omega = \left\{ \omega_1, \omega_2, \ldots, \omega_d \right\}$ satisfies $\left( \gamma, \nu \right)$-Diophantine condition iff
    $\,\left| \bm k \cdot \bm \omega \right| \ge \gamma \cdot {\norm{\bm k}_1}^{-\nu}$, $\forall \bm k \in \Z^d$, $\bm k \neq \bm 0$,
    for some $\gamma>0, \nu>0$.
\label{def:Diophantine}
\end{definition}

\begin{definition}[$\varepsilon$-neighborhood condition]
$\left( \qp \right) \in \R^d \times \R^d$ of an integrable system satisfies $\varepsilon$-neighborhood condition if there exists $\bm I^* \in \R^d$,
such that $\bm \omega\left( \bm I^* \right)$ satisfies the $\left( \gamma, \nu \right)$-Diophantine condition (Def.\ref{def:Diophantine}),
and $\norm{\bm I\left( \qp \right) - \bm I^*}_2 \le c \cdot {\left| \log{\varepsilon} \right|}^{-\nu-1}$ for some $\varepsilon$ independent constant $c$ (defined in the Appendix) with
$\bm I \left( \qp \right)$ being the actions of the system.
\label{def:epsNbhCond}
\end{definition}

With these preparations, we see the action and angle variables $\Iphi$ form a new coordinate system alternative to $\qp$ (note even if the system is not integrable and/or time is no longer continuous, one is still free to perform any canonical coordinate transformation; it's just doing so may or may not reveal structured dynamics any more). In fact, they give finer estimates of the prediction error:

\begin{theorem}[GFNN's long-time prediction error in actions and angles]\label{thm:GFNN_action_angle}
    Consider an integrable Hamiltonian system written in action-angle variables,
    whose exact time-h flow map corresponds to generating function $S_h(\cdot, \cdot)$.
    Predict its trajectory using \GFNN{} with learned generating function $S_h^\theta \left( \cdot, \cdot \right)$ in a bounded data domain $\Dc = \Dc_1 \times \T^d \subseteq \R^d \times \T^d$.
    $\exists\,\varepsilon>0, \rho>0$, such that if the learned generating function ${S}_h^\theta$ (extended in a complex neighborhood of $\Dc$) is analytic and satisfies 
        \begin{align*}
            \sum_{i=1,2} \norm{\partial_i S_h^\theta \left( \cdot,\cdot \right) - \partial_i S_h \left( \cdot,\cdot \right)}_{\infty} \le C_1 \varepsilon,
        \end{align*}
    for some $\varepsilon$ independent constant $C_1$,
    where the $L^\infty$ norm is defined over the $\varepsilon$ independent complex neighborhood $\Bc_\rho\left(\Dc\right)$ of $\Dc$,
    then, $\forall \left( \Iphit{0} \right) \in \mathcal{D}$ that satisfies $\varepsilon$-neighborhood condition (Def.\ref{def:epsNbhCond}), the predicted sequence 
     $\left( \Iphit{0} \right), \left( \Iphit{1}, \right), \ldots$ generated by \GFNN{} satisfies
    \begin{align}
    \left\{
        \begin{aligned}
        & \norm{\bm I_n -  \bm I(0)}_2 \le C \cdot \varepsilon, \\
        & \norm{\bm \varphi_n -  \bm \varphi(n h)}_2 \le C \cdot \left( n h \right) \cdot \varepsilon, \\
        \end{aligned}
    \right.
    \quad \forall n \le h^{-1} \varepsilon^{-1},
    \end{align}
    for some constant $C$.
\end{theorem}

The intuition behind the proof of \cref{thm:GFNN_action_angle} (which is in Appendix) is the following: the predicted dynamics $(\bm I_n, \bm \varphi_n)$ and the true dynamics $(\bm I(nh), \bm \varphi(nh))$ deviate because each step of the prediction introduces some error due to inaccurate $S_h^\theta$, but these errors accumulate in a very delicate way; in fact, earlier errors cannot be amplified too much in order for a linear bound to exist. The key reason, as the proof will recover, is that $\bm I_n$'s dynamics is mostly just oscillatory in time. We show this by decomposing the predicted dynamics into a macroscopic part plus microscopic oscillations. The macroscopic part can be proved to correspond to a barely changing action. The microscopic part, on the other hand, does not accumulate.

Specifically, we introduce a carefully-chosen near-identity canonical coordinate change $\Tc : \left[ \Iphi \right] \mapsto \left[\Jtheta \right]$, $\Tc \approx id+\mathcal{O}(\varepsilon)$, and show that the new variables $[\bm J_n, \bm\theta_n]$ describe, roughly, the macroscopic part of the predicted dynamics.

We then prove, when compared to the true dynamics,
    \begin{align*}
    \left\{
    \begin{aligned}
    & \norm{\bm I(n h) - \bm J_n}_2 = \Oc\left( \epsilon \right),\\
    & \norm{\bm \varphi (n h) - \bm \theta_n}_2 = n h \cdot \Oc (\varepsilon), \\
    \end{aligned}
    \right.
    \quad \forall nh = \Oc\left( \varepsilon^{-1} \right).
    \end{align*}
Since $\Tc$ is near-identity, $[\bm J_n,\bm\theta_n] = [\bm I_n,\bm\varphi_n]+\mathcal{O}(\varepsilon)$ for all $n$, and the triangle inequality then completes the proof. $\square$

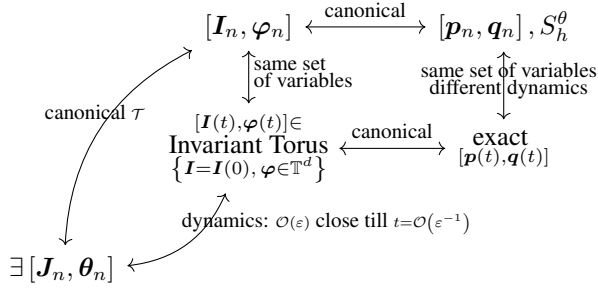
\begin{figure}[ht]
\vskip 0.2 in
\begin{center}
\begin{tikzcd}[nodes in empty cells, column sep=1.2em,row sep=2.0em]
    & \left[ \Iphit{n} \right] \arrow[rr, leftrightarrow, "\clap{canonical}"] \arrow[d, leftrightarrow, "\substack{\rlap{same set }\\\rlap{of  variables}}"] &  & \left[ \qpt{n} \right], S_h^\theta \arrow[d, "\substack{\clap{same set of variables}\\\clap{different dynamics} }", leftrightarrow] \\
    & \substack{\left[ \bm I(t), \bm \varphi(t) \right] \in \\ \clap{Invariant Torus}\\\left\{\bm I=\bm I(0),\,\bm \varphi \in \T^d\right\}} \arrow[rr, leftrightarrow, "\clap{canonical}"] &  & \substack{\clap{exact }\\\left[ \bm p(t), \bm q(t) \right]} \\
    \exists \left[ \Jthetat{n} \right] \arrow[ur, bend right, leftrightarrow, "\substack{
    \rlap{dynamics: $\Oc\left( \varepsilon \right)$ close till $t=\Oc\left( \varepsilon^{-1} \right)$}}"]
    \arrow[uur, bend left, leftrightarrow, "\substack{\clap{canonical $\Tc$}}"] & & &
\end{tikzcd}
\caption{Main components in the proof of linear error growth.}
\label{proof_diagram}
\end{center}
\vskip -0.2 in
\end{figure}

Now we can relate the error bound in  \cref{thm:GFNN_action_angle} back to that for the original variables $\qp$. The big picture is summarized by Fig.\ref{proof_diagram}. In the end, our theory only requires the existence of action and angle variables, and no knowledge about how to find the actions and angles is needed.

\begin{theorem}[linear growth of GFNN long-time prediction error]
    Consider an integrable Hamiltonian system whose exact solution is denoted by $\bm p(t), \bm q(t) \in \R^d$. Denote by $S_h \left( \cdot, \cdot \right)$ the generating function corresponding to its exact time-$h$ flow map. Consider predicting its trajectory using \GFNN{} with learned generating function $S_h^\theta \left( \cdot, \cdot \right)$ in a bounded data domain $\Dc \subseteq \R^d \times \R^d$. 
    $\exists\,\varepsilon>0, \rho>0$, such that if the learned generating function ${S}_h^\theta$ (extended in a complex neighborhood of $\Dc$) is analytic and satisfies 
        \begin{align}
            \sum_{i=1,2} \norm{\partial_i S_h^\theta \left( \cdot,\cdot \right) - \partial_i S_h \left( \cdot,\cdot \right)}_{\infty} \le C_1 \varepsilon,
            \label{eq:neededNNaccuracy}
        \end{align}
    for some $\varepsilon$ independent constant $C_1$,
    where the $L^\infty$ norm is defined over the $\varepsilon$ independent complex neighborhood $\Bc_\rho\left(\Dc\right)$ of $\Dc$,
    then, $\forall \left( \qpt{0} \right) \in \mathcal{D}$ that satisfies $\varepsilon$-neighborhood condition (Def.\ref{def:epsNbhCond}) with nonlinear frequency $\bm\omega(\cdot)$ being given by $S_h$,
    the predicted sequence 
     $\left( \qpt{0} \right), \left( \qpt{1}, \right), \ldots$ generated by \GFNN{} satisfies
    \begin{align}
    \left\{
    \begin{aligned}
    \norm{\bm p_n - \bm p\left( n h \right)}_2
         \le C \cdot \left( n h \right) \cdot \varepsilon, \\
    \norm{\bm q_n - \bm q\left( n h \right)}_2
         \le C \cdot \left( n h \right) \cdot \varepsilon, \\
    \end{aligned}
    \right.
    \quad \forall n \le h^{-1} \varepsilon^{-1},
    \end{align}
    for some constant $C>0$.
    \label{thm:main}
\end{theorem}

\begin{remark}
    It is known that neural networks can approximate functions and their derivatives with any precision; see e.g., the classical work \cite{hornik1990universal} and a more recent discussion \cite{yarotsky2017error}. \eqref{eq:neededNNaccuracy} can thus be attained.
    \label{rmk:universalApprox}
\end{remark}

\begin{remark}
    The integrability assumption in Thm.\ref{thm:main} is nontrivial, however reasonable. This is because it rules out the possibility of a positive Lyapunov exponent, which by definition indicates that a deviation between two trajectories can exponentially grow in time (e.g., \citealp{alligood1996chaos}). Naturally, if the latent system does have a positive Lyapunov exponent, then in general one should not expect a linearly growing prediction error, as an arbitrarily small approximation error, even if it's just made in one step, can be exponentially amplified.
    
    A simple illustration of this is a Hamiltonian system $\dot{x}=y, \dot{y}=x$, which is not integrable due to noncompactness (not even chaos). It has a Lyapunov exponent of $+1$. Consider predictions based on approximation $\dot{x}=y+\delta_x, \dot{y}=x+\delta_y$, then no matter how small $\delta_x$ and $\delta_y$ are, the difference between its solution and the original one grows like $\exp(t)$ except for measure zero $\delta_x$ and $\delta_y$ values.
    \label{rmk:chaosPositiveLyapunov}
\end{remark}

As a comparison, if the prediction map is not symplectic, either due to nonsymplectic numerical integration, or because the learned vector field is no longer Hamiltonian, local prediction error (in each step) may get amplified and long time prediction error may grow exponentially:

\begin{theorem}
	Consider the latent dynamics $\dot{x}=f(x)$ and its prediction via an Euler integration of the learned vector field $x_{i+1}=x_i+h \tilde{f}(x_i)$, with consistent initial condition $x(0)=x_0$. Assume $f$ is $L$-Lipschitz continuous, $\mathcal{C}^1$, the learned vector field is accurate up to $\delta$ in the sense that $\|\tilde{f}-f\|_\infty \leq \delta$, and the prediction remains bounded. Then the accuracy of the prediction at time $T=nh$ satisfies
	\[
		\|x(T)-x_n\| \leq \frac{\exp(L T)-1}{L}(\delta+L h/2).
	\]
	\label{thm:genericPredictionErrorBound}
\end{theorem}

\begin{proof}
See Appendix.
\end{proof}

\begin{remark}
	This exponential growth with $T$ (and $n$) is not an overestimation. A simple example that attains it is $f(x)=x$ and $\tilde{f}(x)=x+\delta$. This is of course because the latent dynamics is structurally bad and does not forgive past errors, but that is exactly our point: when the latent dynamics has specific structures such as being a symplectic flow, utilizing those structures in the prediction could lead to much better controlled accumulation of errors.
\end{remark}


\begin{remark}
In the context of learning dynamics from data, two sources contribute to the difference between $f$ and $\tilde{f}$. One is approximation error, for instance of the neural network; the other is because one doesn't have an oracle about the latent vector field $f$ but only its approximation from the data, for example $f(x_i)\approx (x_{i+1}-x_i)/h$. A map based approach doesn't directly use $f$ and thus can avoid the latter error, and it doesn't have numerical integration errors in the next phase of predictions either. A neural ODE type treatment \cite{zhong2019symplectic} can avoid the latter error too, but integration errors in the prediction phase remain (unless computationally expensive small steps are used).
\end{remark}

\section{Experiments}\label{sec:experiments}
Let's now systematically (within the page limit) investigate the empirical performances of GFNN. It was conjectured that invariant sets of a smooth map with a dense trajectory are typically either periodic, quasiperiodic\footnote{A function $f(t)$ is quasiperiodic if $\exists$ some constants $n\in\mathbb{Z}^+$, $\Omega_1,\cdots,\Omega_n\in \mathbb{R}$, and some function $F$ 1-periodic in each argument, s.t., $f(t)=F(\Omega_1 t,\cdots,\Omega_n t)$. An integrable system's solution is quasiperiodic if LCM$(\omega_1(\bm I), \cdots, \omega_d(\bm I))$ doesn't exist (see Rmk.\ref{rmk_invariantTori} for $\omega_i(\bm I)$); otherwise it is periodic.
\label{note:quasiperiodicity}}, or chaotic \cite{sander2015many}. Thus, Sec.\ref{sec:2body}-\ref{sec:Results:standardMap} will study classical examples that respectively correspond to periodic, quasiperiodic+chaotic, quasiperiodic, and quasiperiodic+chaotic cases. We'll see smaller and linearly growing errors of GFNN in both periodic and quasiperiodic cases, even when the latent system is not integrable. In chaotic cases, GFNN will also exhibit pleasant statistical accuracy.

VFNN stands for: learning the Vector Field via a Neural Network (without caring about the Hamiltonian structure).

Details of data preparation and training are in \cref{appendix:experimental_details}.


\subsection{An Integrable and Separable Hamiltonian: 2-Body Problem}
\label{sec:2body}

Consider the motion of 2 gravitationally interacting bodies. Letting their distance be $\bm q(t)$ and the corresponding momentum be $\bm p(t)$, the problem can be equivalently turned into (after unit normalization) an ODE system governed by
\[
 H\left( \qp \right) = \|\bm p\|_2^2/2 - 1/\norm{\bm q}_2.
\]
Despite its high nonlinearity, this is an integrable system. Analytical solutions known as Keplerian orbits exist and are periodic in bounded cases. Each solution is described by important physical quantities known as orbital elements, which include semi-major axis and eccentricity, that characterize the shape of the elliptic orbit. 
As shown in \cref{fig:2body}, GFNN outperforms other methods and keeps the errors of semi-major axis and eccentricity small and bounded, which is consistent with \cref{thm:GFNN_action_angle} because semi-major axis and eccentricity are functions of actions known as Delaunay variables \cite{morbidelli2002modern}. The advantage of GFNN can also been seen in the original variables (e.g., $\bm q$), and the zoomed-in plots in row 2 show that the next two top performers are SRNN (seq\_len=2) and SRNN (seq\_len=5); SympNets has notably larger error in the orbital phase but its accuracy in the orbital shape is actually comparable to SRNN.

\begin{figure}[ht]
\vskip 0.2in
\begin{center}
     \includegraphics[width=\linewidth]{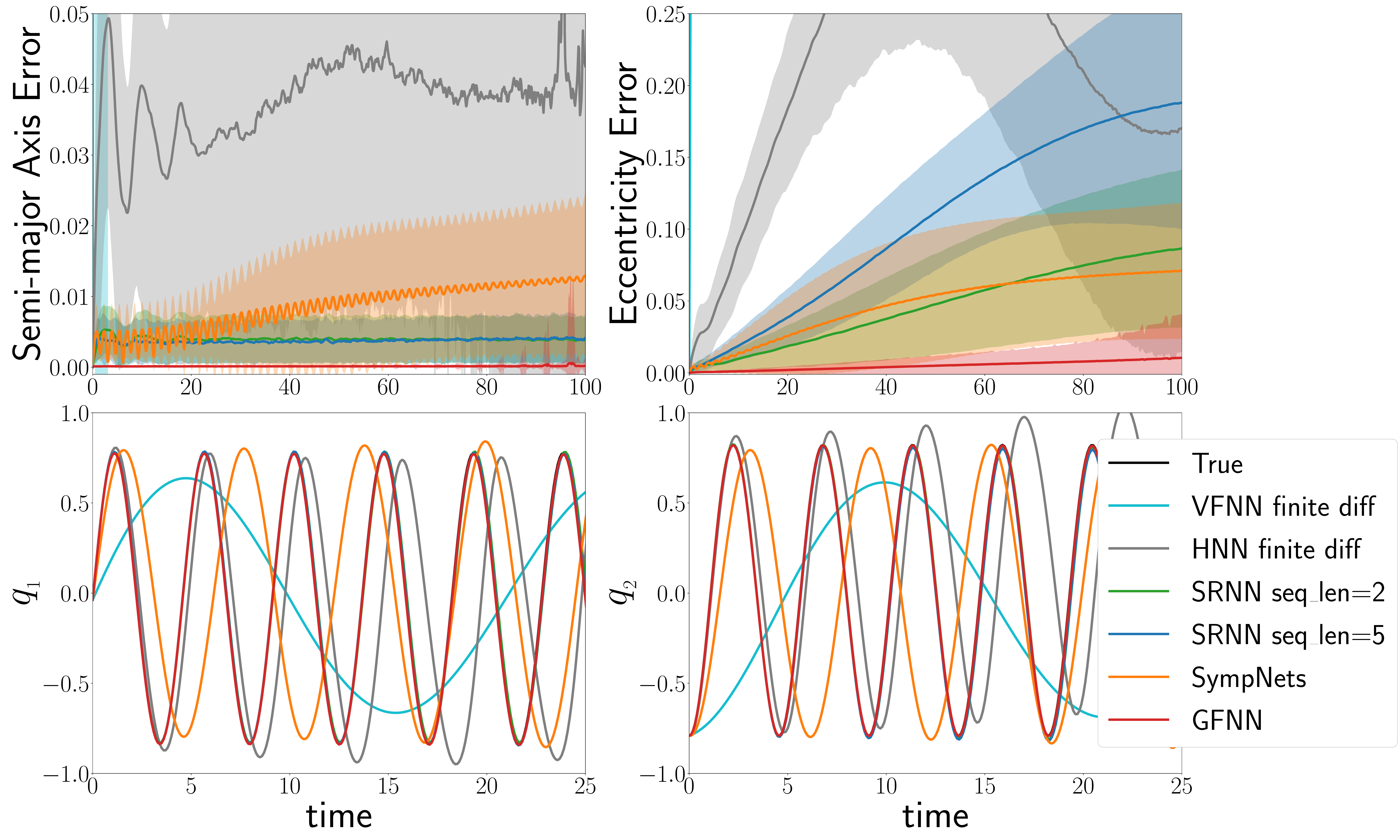}
     \vskip -10pt
     \caption{Comparison of $2$D Keplerian orbit predicted by different methods. The 1st row is the error growth of two variables of physical importance, namely semi-major axis and eccentricity (for this problem, their true values are both constants). Mean values of prediction errors starting from 1,000 i.i.d. initial conditions are plotted with shades corresponding to 1 standard deviation. The 2nd row zooms in the position variables of one predicted sequence (in $q_1$ and $q_2$ respectively). Data sequences are prepared with time step $0.1$.}\label{fig:2body}
\end{center}
\vskip -0.2in
\end{figure}

\subsection{A Non-integrable but Separable Hamiltonian System: H\'enon-Heiles}
\label{sec:HenonHeiles}

The H\'enon-Heiles system describes the motion of stars around a galactic center \cite{henon1964applicability}. It is a classical non-integrable system with very complex dynamics, governed by Hamiltonian
\[
H\left( p_1, p_2, q_1, q_2 \right) 
= \frac{\left( p_1^2 + p_2^2 \right)}{2} + \left( \frac{q_1^2 + q_2^2}{2} + q_1^2 q_2 + \frac{q_2^3}{3}\right).
\]
Both chaotic and (quasi)-periodic solutions exist. Initial conditions corresponding to higher energy (i.e., $H$'s value) are more likely to be chaotic. We investigate \GFNN's performance in both cases.

\subsubsection{Non-periodic but regular motions}
For a non-chaotic initial condition, numerically observed was that the long time prediction error of GFNN still grows linearly even though the latent system is no longer integrable; see \cref{fig:regular_hh}. Notably, SRNN also exhibits linear error growth (although at a higher rate), and this is consistent with our intuition as SRNN also learns a symplectic evolution map (indirectly via the symplectic integration of a Hamiltonian to-be-learned). HNN, on the other hand, has exponentially growing error which quickly saturates to maximum values (due to boundedness of trajectories).

\begin{figure}[ht]
\vskip 0.2in
\begin{center}
     \includegraphics[width=\linewidth]{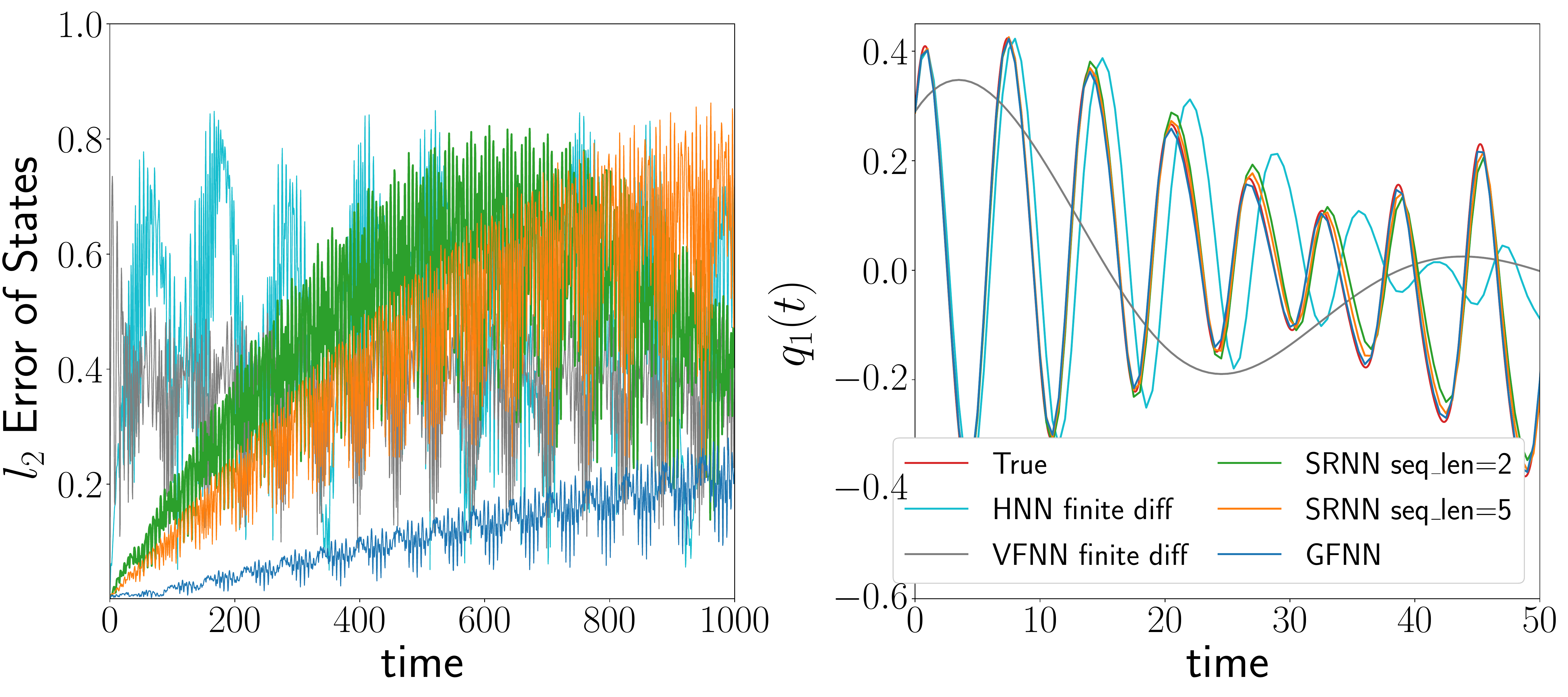}
     \vskip -10pt
     \caption{Error of predictions of a quasi-periodic trajectory with energy level near $\frac{1}{12}$ of the H\'enon-Heiles system. Data sequences are prepared with time step $0.5$.}\label{fig:regular_hh}
 \end{center}
\vskip -0.2in
\end{figure}

\GFNN's linear error growth despite non-integrability is due to the existence of mathematical objects known as KAM-tori (e.g., \citealp{poschel1982integrability}). They correspond to part of the phase space where dynamics are topologically equivalent to integrable ones. A by-product is, solutions in this region are either quasiperiodic or periodic (see Footnote \ref{note:quasiperiodicity}).

\subsubsection{Dynamics in chaotic sea}
To visualize the prediction of chaotic dynamics, which take place in 4D, we use the standard tool of Poincar\'e section, which plots where an orbit intersects with a 2D slice of the 3D constant-energy manifold. \cref{fig:poincare_section_hh_chaotic} shows the Poincar\'e section produced by predictions of different methods, based on the same initial condition that leads to chaotic motion via the latent dynamics.
The true chaotic motion is ergodic on a submanifold of the phase space, and when restricting to the  Poincar\'e section, it gives intersections that are dense in a subset known as the chaotic sea. Therefore, the shapes of the dense area and the holes inside it (often corresponding to regular islands on which motions are (quasi)-periodic) are indicators of the prediction accuracy. Among methods tested in \cref{fig:poincare_section_hh_chaotic}, only VFNN didn't produce a pattern similar to the truth. Quantitative comparisons are conducted by comparing the empirical distributions of points on the Poincar\'e section, and KL divergences between their marginals and the truth are annotated along with the histograms. GFNN has the smallest errors.



\begin{figure}[ht]
\vskip 0.2in
\begin{center}
     \includegraphics[width=\linewidth]{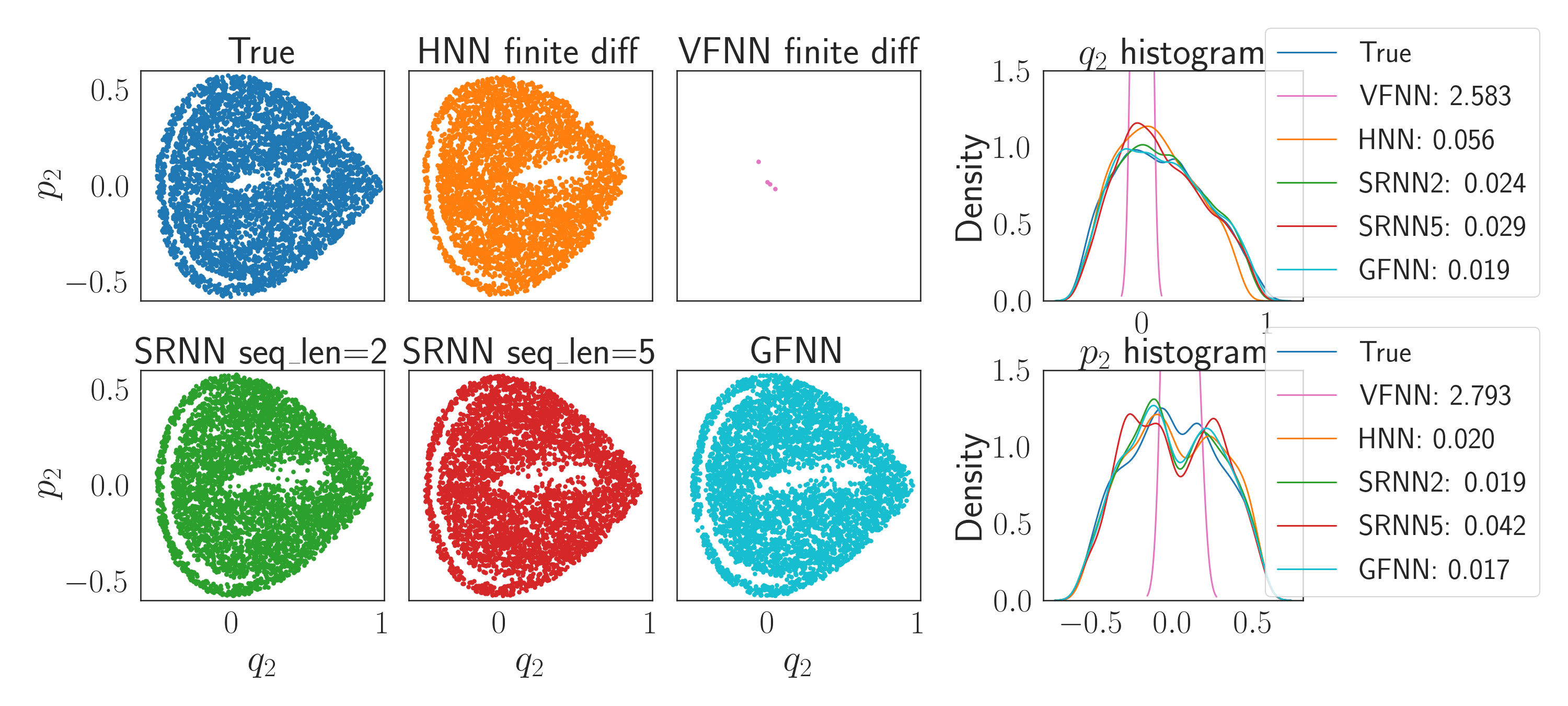}
     \vskip -10pt
     \caption{Quantifying the statistical accuracy in predicting a chaotic orbit of the H\'enon-Heiles system. Left 3 columns: Poincar\'e section; right column: marginal distributions and their KL divergences from the truth. The plotted orbit corresponds to energy $\frac{1}{6}$; Poincar\'e section is given by $q_2(t)$, $p_2(t)$ at $q_1(t)=0$. Data sequences are prepared with time step $0.5$.}
     \label{fig:poincare_section_hh_chaotic}
\end{center}
\vskip -0.2in
\end{figure}

\subsection{A Non-integrable and Non-separable Hamiltonian: \\
Planar Circular Restricted 3-Body Problem (PCR3BP)}
\label{sec:PCR3BP}

PCR3BP is a special case of the gravitational 3-body problem. In addition to a co-planar restriction, it assumes two bodies massive and the third infinitesimal, which models settings like mission design for a space shuttle near Earth and Moon \cite{koon2000heteroclinic}, and understanding a planet's motion around binary stars \cite{li2016uncovering, quarles2020orbital}.
Its Hamiltonian takes the form
\[
\begin{split}
& H\left( \qp \right) = \frac{p_1^2 + p_2^2}{2} + p_1 q_2 - p_2 q_1 \\
& \qquad - \frac{1-\mu}{\norm{\left(q_1+\mu,q_2\right)}_2} - \frac{\mu}{\norm{\left(q_1+\mu-1, q_2\right)}_2},
\end{split}
\]
with $\mu \in (0, 1)$ a constant mass parameter. Note it cannot be written as $K(\bm p)+V(\bm q)$, hence nonseparable.

In order to focus on comparing with SOTA methods for trajectory accuracy, we predict solutions in the nearly-integrable (non-chaotic) regime of PCR3BP; see Fig.\ref{fig:PCR3BP}. GFNN still has the smallest error among those experimented, and its growth is again linear. SRNN typically performs the best among tested existing approaches, but its published version loses symplecticity in this case due to non-separability\footnote{A possible remedy based on our nonseparable symplectic integrators \cite{tao2016explicit} was mentioned in SRNN as a future direction. This remedy is implemented in a concurrent work \cite{xiong2021nonseparable}, which successfully reduces the error of predicting nonseparable dynamics to the level of SRNN for separable dynamics.\label{ftnt:SRNN_nonseparable}}, and its accuracy deteriorated. Note also that for methods that learn, in the separable case, $V(\bm q)$ in the Hamiltonian or $\nabla V(\bm q)$ in the vector field, now they cannot just do so but have to learn the entire $H(\qp)$ in doubled dimensions.

\begin{figure}[ht]
\vskip 0.2in
\begin{center}
     \includegraphics[width=\linewidth]{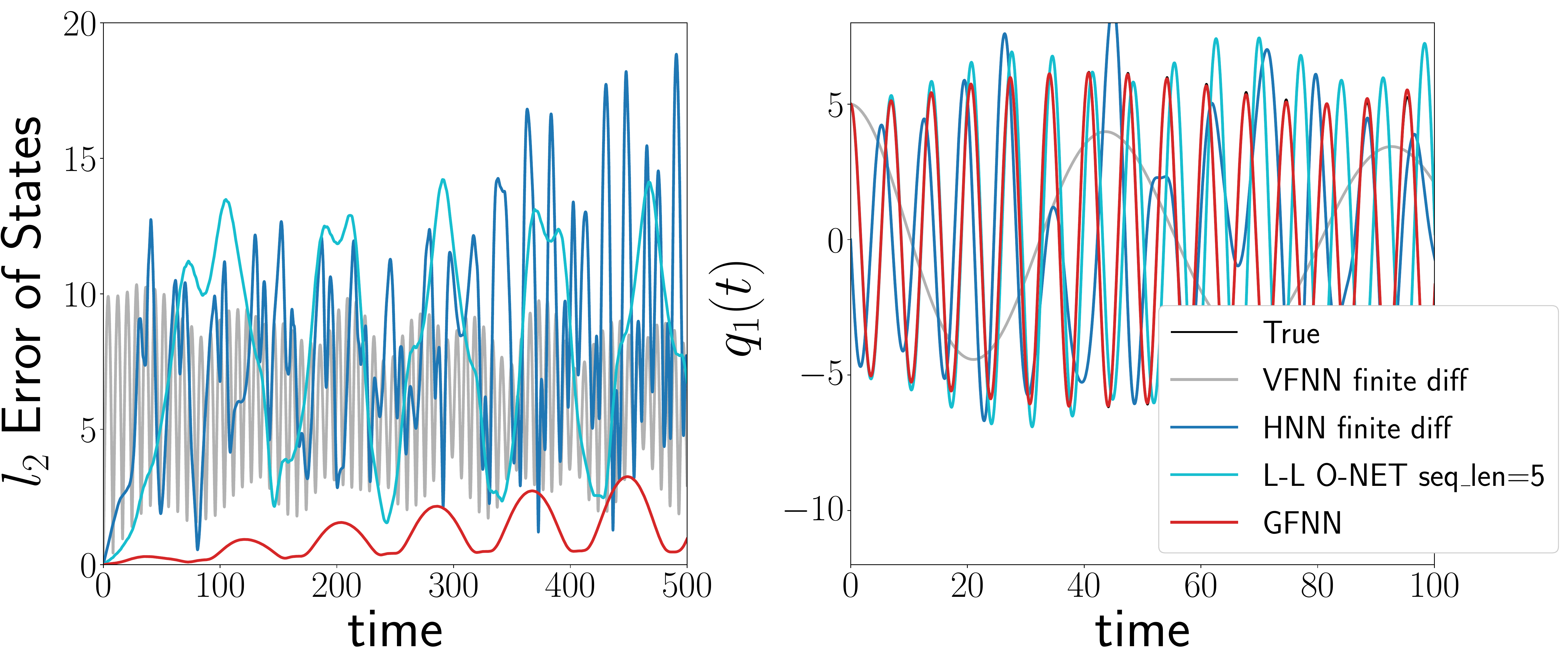}
     \vskip -10pt
     \caption{Comparison of PCR3BP orbit predicted by different methods. L-L O-NET (in SRNN paper) is selected instead of SRNN as the Hamiltonian is not separable. Data sequences are prepared with time step $0.1$.
     }\label{fig:PCR3BP}
 \end{center}
\vskip -0.2in
\end{figure}

\subsection{A Discrete-time Non-(Smooth-)Hamiltonian System: the Standard Map}
\label{sec:Results:standardMap}

The standard map is a classical model in accelerator physics. It is a chaotic system whose statistical property is (relatively) well understood. It is a symplectic map in 2D given by
\begin{equation}
\begin{cases}
p_{n+1} &= p_n + K \sin(\theta_n), \\
\theta_{n+1} &= \theta_n + p_{n+1}.
\end{cases}
\label{eq:standardMap}
\end{equation}
The dynamics is genuinely in \emph{discrete} time, as no smooth Hamiltonian ODE can produce a flow map like it\footnote{This is because autonomous Hamiltonian systems in 2D are never chaotic (the Hamiltonian itself is a 1st integral) but the standard map is chaotic.}. $K$ is a positive constant that controls the strength of nonlinearity, and it has been estimated that the region of initial conditions leading to chaos has size increasing with $K$ \cite{chirikov1979universal}.

Methods based on vector fields (e.g., VFNN) or Hamiltonian (e.g., HNN, SRNN) are not very suitable for this prediction task because there is no latent continuous (Hamiltonian) dynamics. One can still apply these methods regardless, for example by using finite differences to construct a fictitious vector field for VFNN and HNN to learn, or just use SRNN without realizing that no Hamiltonian will be able to produce the training data.
Their results (obtained using $h=1$) will be compared with those of GFNN, which is still applicable here as it directly learns evolution maps.


\cref{fig:chaotic_sea_standard_map} illustrates the predicted evolutions of a fixed initial condition in the chaotic sea (of the true dynamics, $K=1.2$) by various methods. Note both $\theta$ and $p$ have been mod $2\pi$ as this quotient compactifies the phase space into the 2-torus without affecting the dynamics (see \cref{eq:standardMap}). Like before, the prediction quality can be inferred from the geometric shape of the set of plotted points, which should match that of the truth (i.e., the latent chaotic sea), and quantitative comparisons can be made using distances between empirical distributions of $p,\theta$ values collected along long time predictions (KL divergences from the truth are provided).


One can see GFNN is the only method that captures the major regular islands (the big holes), but even GFNN does not capture the minor regular islands well. The standard map seems to be a challenging problem; HNN and SRNN did not manage to reproduce any chaotic motion, and VFNN completely distorted the chaotic sea.

\cref{fig:regular_island_standard_map} on the other hand illustrates predictions in regular islands (of the true dynamics, $K=0.6$).
The two (not three, note periodic boundary conditions) elliptical shapes near $p \approx \pi$ and $\theta \approx 0, \pi$ correspond to quasiperiodic orbits, and GFNN is the only one that captures them: the exact trajectory is jumping back and forth between two islands, so does GFNN's prediction, while other methods tend to produce continuous trajectories without capturing the jumps.

\begin{figure}[ht]
\vskip 0.2in
\begin{center}
     \includegraphics[width=\linewidth]{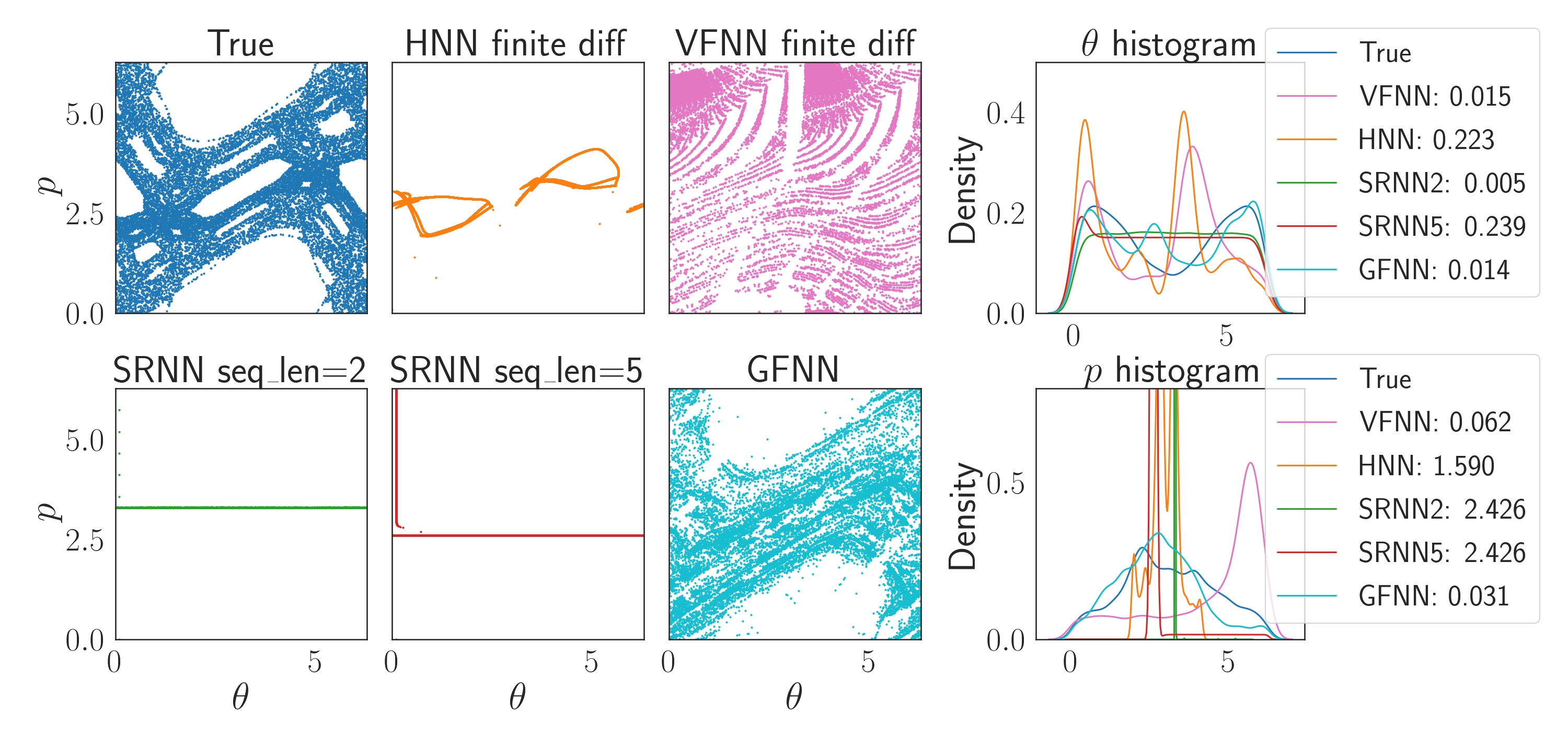}
     \vskip -10pt
     \caption{Predict a chaotic orbit of the standard map. Left 3 columns: the predicted orbit in phase space; right column: marginals of its empirical measure and their KL divergences from the truth.}\label{fig:chaotic_sea_standard_map}
\end{center}
\vskip -0.2in
\end{figure}

\begin{figure}[ht]
\vskip 0.2in
\begin{center}
     \includegraphics[width=\linewidth]{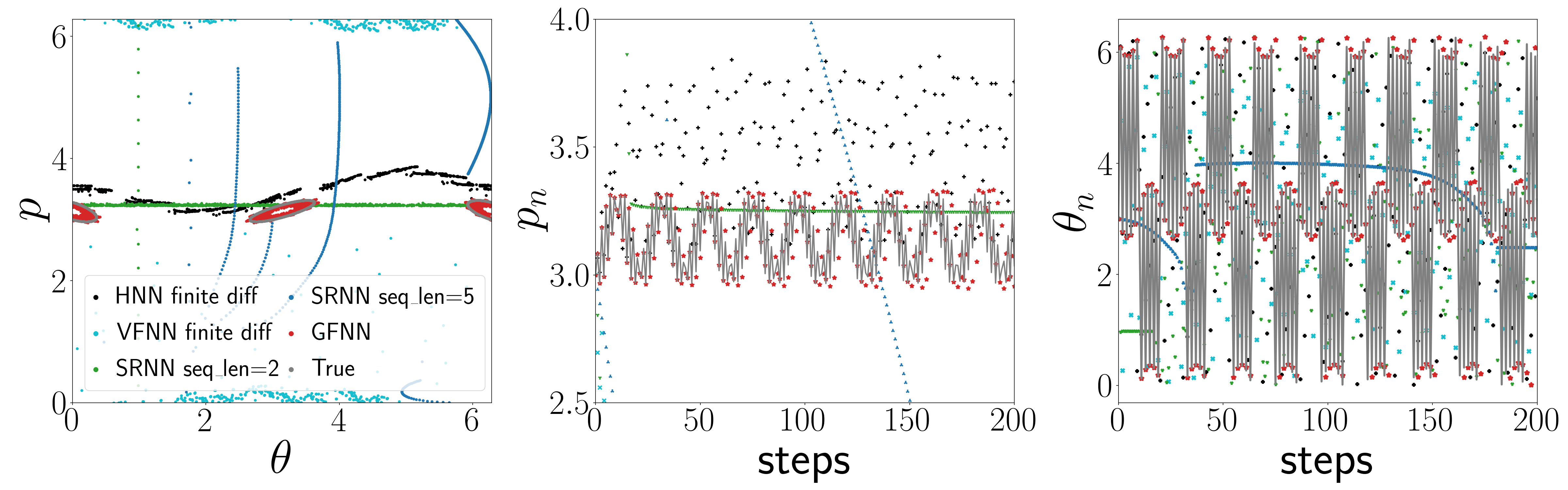}
     \vskip -10pt
     \caption{Predict a regular orbit of the standard map. The 1st plot is orbits predicted by various methods in phase space. The 2nd and 3rd plots shows how their two coordinates change with time.}\label{fig:regular_island_standard_map}
\end{center}
\vskip -0.2in
\end{figure}


\section*{Acknowledgements}
The authors thank the anonymous reviewers who helped significantly improve the quality of this article, and Minshuo Chen for helpful discussions. We are grateful for the partial support by NSF DMS-1847802 and ECCS-1936776.

\bibliographystyle{icml2021}
\bibliography{ChenTao2021}

\newpage
\appendix
\section{Mathematical Proof}

\subsection{Proof of \cref{thm:genericPredictionErrorBound}}
\begin{lemma}
	Consider $\dot{x}=f(x)$ and $\dot{y}=f(y)$, where $f$ is $L$-Lipschitz continuous. Then
	\begin{equation}
		\|x(t)-y(t)\| \leq \exp(Lt) \|x(0)-y(0)\|.
		\label{eq:continuityInIC}
	\end{equation}
    \label{lemma:continuityInIC}
\end{lemma}

\begin{proof}
	Since
	\begin{align*}
		x(t) &= x(0)+\int_0^t f(x(\tau)) d\tau \\
		y(t) &= y(0)+\int_0^t f(y(\tau)) d\tau ,
	\end{align*}
	triangular inequality and Lipschitz continuity give
	\begin{align*}
		& \|x(t)-y(t)\| \\
		&\leq \|x(0)-y(0)\| + \int_0^t \|f(x(\tau))-f(y(\tau))\| d\tau  \\
		&\leq \|x(0)-y(0)\| + \int_0^t L \|x(\tau)-y(\tau)\| d\tau.
	\end{align*}
	Gronwall inequality thus gives \cref{eq:continuityInIC}.
\end{proof}

\begin{lemma}
	Consider $\dot{x}=f(x)$, with $x(0)=x_0$ and $L$-Lipschitz continuous $f$. Then
	\begin{equation}
		\|x(h)-x_0\| \leq \frac{\exp(L h)-1}{L}\|f(x_0)\|.
		\label{eq:deviationFromInitialCondition}
	\end{equation}
	(Note $\frac{\exp(L h)-1}{L}=\mathcal{O}(h)$.)
	\label{lemma:deviationFromInitialCondition}
\end{lemma}

\begin{proof}
	Note
	\begin{align*}
		x(h) &= x_0+\int_0^h f(x(\tau)) d\tau \\
		&= x_0 + \int_0^h f(x(\tau))-f(x_0) + f(x_0) d\tau.
	\end{align*}
	Let $D(t):=x(t)-x_0$. Then triangular inequality and Lipschitz continuity of $f$ give
	\[
		D(h) \leq \int_0^h L D(\tau) + \|f(x_0)\| d\tau
	\]
	Gronwall lemma thus yields
	\[
		D(h) \leq \exp(L h) D(0) + \frac{\exp(L h)-1}{L}\|f(x_0)\|.
	\]
	Since $D(0)=0$, \cref{eq:deviationFromInitialCondition} is proved.
\end{proof}

\begin{proof}[Proof of \cref{thm:genericPredictionErrorBound}]\label{proof:genericPredictionErrorBound}
	Let $E_n=\|x(nh)-x_n\|$ denote the prediction accuracy, and $\phi_{x_0}^h$ be the $h$-time flow map of the latent dynamics, i.e., $\phi_{x_0}^h := x(h)$ where $x(\cdot)$ satisfies $\dot{x}=f(x)$ subject to $x(0)=x_0$. Then
	\[
		x((n+1)h)-x_{n+1} = x((n+1)h)-\phi_{x_n}^h+\phi_{x_n}^h-x_{n+1},
	\]
	and therefore
	\[
		E_{n+1} \leq \|x((n+1)h)-\phi_{x_n}^h\| + \|\phi_{x_n}^h-x_{n+1}\|.
	\]
	The first term is exactly $\|\phi_{x(nh)}^h-\phi_{x_n}^h\|$, and by \cref{lemma:continuityInIC}, it is bounded by
	\[
		\|\phi_{x(nh)}^h-\phi_{x_n}^h\| \leq \exp(L h) \|x(nh)-x_n\| = \exp(L h) E_n.
	\]
	For the second term, Taylor expansion gives
	\[
		\phi_{x_n}^h = x_n + h f(x_n) + h^2/2 f'(\phi_{x_n}^\xi)f(\phi_{x_n}^\xi)
	\]
	for some $\xi\in[0,h]$, and therefore
    \begin{align*}
        \begin{split}
            \|\phi_{x_n}^h-x_{n+1}\| & = \|h (f(x_n)-\tilde{f}(x_n)) + h^2/2 f'(\phi_{x_n}^\xi)f(\phi_{x_n}^\xi)\| \\
            & \leq h\delta + h^2/2\|f'(\phi_{x_n}^\xi)\|\|f(\phi_{x_n}^\xi)\|.
        \end{split}
    \end{align*}
	Note $\|f'\| \leq L$ as $f$ is $\mathcal{C}^1$ and $L$-Lipschitz. For the $f(\phi_{x_n}^\xi)$ factor, note \cref{lemma:deviationFromInitialCondition} gives
	\[
		\|\phi_{x_n}^\xi - x_n \| \leq \frac{\exp(L \xi)-1}{L} \|f(x_n)\|,
	\]
	and therefore
	\begin{align*}
		\|f(\phi_{x_n}^\xi)\| &= \|f(x_n) + f(\phi_{x_n}^\xi) - f(x_n)\| \\
        & \leq \|f(x_n)\| + \|f(\phi_{x_n}^\xi) - f(x_n)\| \\
        & \leq \|f(x_n)\| + L \|\phi_{x_n}^\xi - x_n \| \\
		&\leq \exp(L\xi) \|f(x_n)\|
	\end{align*}
	Since $0\leq\xi\leq h$, $\exp(L\xi)$ is bounded. Moreover, $f(x_n)$ is bounded because $f$ is Lipschitz and therefore continuous and $x_n$ is assumed to be bounded. Therefore, there exists constant $C$ such that
	\[
		\|f'(\phi_{x_n}^\xi)\|\|f(\phi_{x_n}^\xi)\| \leq C
	\]
	Summarizing both terms, we have
	\[
		E_{n+1} \leq E_n \exp(Lh) + h\delta + Ch^2/2.
	\]
	Mathematical induction thus gives
	\begin{align*}
        E_N &\leq E_0 {\exp(Lh)}^N \\
        & + \left({\exp(Lh)}^{N-1}+{\exp(Lh)}^{N-2}+\cdots+1\right) (h\delta + Ch^2/2) \\
		&= E_0 \exp(LT) + \frac{\exp(LT)-1}{\exp(Lh)-1} (h\delta + Ch^2/2) \\
		&\leq E_0 \exp(LT) + \frac{\exp(LT)-1}{L h} (h\delta + Ch^2/2) \\
		&= \frac{\exp(LT)-1}{L} (\delta + Ch/2).
	\end{align*}
\end{proof}

\subsection{Proof of \cref{thm:GFNN_action_angle}}

\begin{definition}[Diophantine condition]
    A frequency vector $\bm \omega = \left\{ \omega_1, \omega_2, \ldots, \omega_d \right\}$ satisfies Diophantine condition if and only 
    there exists positive constants $\gamma, \nu$ such that $\bm \omega$ satisfies $\left( \gamma, \nu \right)$-Diophantine condition.
\end{definition}

\begin{definition}[$\left( \gamma, \nu \right)$-Diophantine set]
    For a set $\Omega \subseteq \R^d$, the corresponding $\left( \gamma, \nu \right)$-Diophantine set is defined as
    \begin{align*}
        \Omega^{*}\left( \gamma, \nu \right) \stackrel{def}{=} \big\{ & \bm \omega \in \Omega: \\
        & \bm \omega \text{\ satisfy } \left( \gamma, \nu \right) \text{-Diophantine condition} \big\}. \\
    \end{align*}
\end{definition}

\begin{definition}[Diophantine set]
    For a set $\Omega \subseteq \R^d$, the corresponding Diophantine set is defined as
    \begin{align*}
        \Omega^{*} \stackrel{def}{=} \union_{\gamma>0, \nu>0} \Omega^*\left( \gamma, \nu \right).
    \end{align*}
\end{definition}

\begin{theorem}\label{thm:lebesgue_measure_diphantine_consts}
    For any bounded domain $\Omega \subseteq \R^d$, there exists $C > 0$, such that the Lebesgue measure of the complementary of $\left( \gamma, \nu \right)$-Diophantine set with $\nu\ge d$ is bounded from above,
    \begin{align*}
        \lambda \left( \Omega \backslash \Omega^* {\left( \gamma, \nu \right)} \right) \le C \cdot \gamma.
    \end{align*}
\end{theorem}

\begin{proof}
See for instance \cite{hairer2006geometric}.
\end{proof}

\begin{theorem}\label{thm:lebesgue_measure_diphantine}
    For any bounded domain $\Omega \subseteq \R^d$, Diophantine frequencies exist \textit{almost everywhere}.
\end{theorem}

\begin{proof}
    Since $\lambda\left( \Omega \backslash \Omega^* \right) \le \lambda \left( \Omega \backslash \Omega^*\left( \gamma, \nu \right) \right)$, $\forall \gamma > 0, \nu > 0$,  \cref{thm:lebesgue_measure_diphantine_consts} gives, $\forall \gamma > 0$,
    \begin{align*}
        \lambda \left( \Omega \backslash \Omega^* \right) \le C \cdot \gamma,
    \end{align*}
    meaning that Diophantine frequencies exist \textit{almost everywhere} in $\Omega$.
\end{proof}

\begin{remark}
    Even Diophantine frequencies exist \textit{almost everywhere} in bounded domain $\Omega \subseteq \R^d$, $\Omega \backslash \Omega^*$ is still an open and dense set in $\R^d$ (see for instance \citealp{hairer2006geometric}).
\end{remark}

\begin{lemma}[Cauchy's inequality]\label{thm:Cauchy_inequalities}
    Suppose that $f$ is a holomorphic function on a closed ball $\overline{\Bc_r \left( \theta^* \right)} \subset \mathbb C$ with $r > 0$.
    If $\abs{f\left( \theta \right)} \le M$ for all $\theta$ on the boundary of $\Bc_r \left( \theta^* \right)$, then for all $n\ge 0$,
    \begin{align*}
        \abs{f^{(n)}\left( \theta^* \right)} \le \frac{n! M}{r^n}.
    \end{align*}
\end{lemma}

\begin{proof}
See for instance \cite{stein2010complex}.
\end{proof}

\begin{definition}[average over angles]
Assume $F(\bm \theta)$ is periodic in each argument, i.e., $F: \T^d \to \R$, then the (angle) average of $F$ is defined as
\begin{align}
    \overline{F} = \frac{1}{(2\pi)^d} \int_{\T^d} F(\bm \theta)\, d \bm \theta.
\end{align}
\end{definition}

\begin{definition}[complex extension of $\T^d$]\label{def:torus_complex_extension}
    The complex extension of $\T^d$ of width $\rho$ is defined as
    \begin{align}
        \Bc_\rho\left( \T^d \right) = \left\{ \bm \theta \in \T^d + i \R^d; \norm{\text{Im} \theta} < \rho \right\}.
    \end{align}
\end{definition}

\begin{definition}
For an analytic function $\bm f(\cdot) = \begin{bmatrix} f_1(\cdot), & \ldots, & f_d(\cdot) \end{bmatrix} \in \C^d$, we define the following norm
\begin{align}\label{def:uniform_norm}
   \norm{\bm f}_{\infty, S} := \sum_{i=1}^d \sup_{x \in S} \abs{f_i(x)}.
\end{align}
\end{definition}

\begin{lemma}\label{lemma:Ti_bound}
    Suppose $\bm \omega \in \R^d$ satisfies the $\left( \gamma, \nu \right)$-Diophantine condition and $G(\bm \theta) \in \R$ is a bounded and analytic function on $\Bc_\rho \left( \T^d \right)$.
    Then, with $\overline{G}$ being the average of $G(\bm \theta)$, the PDE
    \begin{align}
    \begin{split}
        D F (\bm \theta) \cdot \bm \omega + G(\bm \theta) = \overline{G} \\
    \end{split}
    \end{align}
    has a unique real analytic solution $F(\cdot)$ with $\overline{F} = 0$. Moreover, for every positive $\delta < \min\left( \rho, 1 \right)$,
    $F$ is bounded on $\Bc_{\rho-\delta}\left( \T^d \right)$ by
    \begin{align*}
    \left\{
        \begin{aligned}
            \norm{F}_{\infty, \Bc_{\rho - \delta}(\T^d)} & \le \kappa_0 \delta^{-\alpha+1} \norm{G}_{\infty, \Bc_{\rho} (\T^d)}, \\
            \norm{\partial_{\bm \theta} F}_{\infty, \Bc_{\rho - \delta} (\T^d)} & \le \kappa_1 \delta^{-\alpha} \norm{G}_{\infty, \Bc_\rho (\T^d)}, \\
        \end{aligned}
    \right.
    \end{align*}
    with $\alpha = \nu + d + 1$ and $\kappa_0 = \nu^{-1} 8^d 2^\nu \nu!$, $\kappa_1 = \nu^{-1} 8^d 2^{\nu+1} \left( \nu+1 \right)!$.
\end{lemma}

\begin{proof}
    See for instance \cite{hairer2006geometric}.
\end{proof}



\begin{lemma}\label{lemma:transform1}
    Consider a nearly integrable system with generating function $S\left( \bm I_0, \bm \varphi_1 \right) = S_0\left( \bm I_0 \right) + \varepsilon S_1 \left( \bm I_0, \bm \varphi_1 \right)$. Suppose $S_0$ and $S_1$ are analytic and bounded in a complex neighborhood of $\Dc_1 \subseteq \R^d$ and $\Dc = \Dc_1 \times \T^d$ respectively.
    Then, there exists a real analytic canonical transformation $\left( \bm J, \bm \theta \right) \leftrightarrow \left( \bm I, \bm \varphi \right)$
    generated by $\Tc \left( \bm J, \bm \varphi \right) = \bm J \cdot \bm \varphi + \varepsilon \Tc_1\left( \bm J, \bm \varphi \right)$, such that
    the generating function in $\Jtheta$ variables takes the form of
    \begin{align}\label{eq:transform1_eq1}
        \begin{split}
            \tilde{S} \left( \bm J_0, \bm \theta_1 \right)
            & = \tilde{S}_0\left( \bm J_0 \right) + \varepsilon^2 \tilde{R}_2 \left( \bm J_0, \bm \theta_1, \varepsilon \right). \\
        \end{split}
    \end{align}
    with $\varepsilon^2 \tilde{R}_2$ being a higher-order perturbation to a new integrable system $\tilde{S}_0$. Moreover, this result is constructive: the transformation $\bm J, \bm \theta \leftrightarrow \bm I, \bm \varphi$ is given by $\Tc$ through
    \begin{align}\label{eq:canonicalT}
        \left\{
        \begin{aligned}
            \bm I_i = \partial_2 \Tc \left( \bm J_i, \bm \varphi_i \right), \\
            \bm \theta_i = \partial_1 \Tc \left( \bm J_i, \bm \varphi_i \right). \\
        \end{aligned}
        \quad \forall i=0,1
        \right.
    \end{align}
    and $\Tc_1$ is the solution to the PDE
    \begin{align}\label{eq:T1_const}
        \partial_2 \Tc_1 \left( \bm J, \bm \varphi \right) \cdot \bm \omega \left( \bm J \right) + G_1 \left( \bm J, \bm \varphi \right) = \overline{G_1} (\bm J),
    \end{align}
    where $\bm \omega(\cdot) = \nabla S_0 \left( \cdot \right)$, $G_1(\bm J, \bm \varphi)=S_1(\bm J, \bm \varphi + h \nabla S_0(\bm J))$, and $\overline{G_1}(\bm J)$ is its angle average.
\end{lemma}

\begin{proof}
    The generating function $S$ in $\Jtheta$ variables ($\tilde{S}$) can be converted in the following using \cref{eq:canonicalT}, 
    \begin{align}\label{eq:taylor_expansion_S}
        \begin{split}
                & \tilde{S} \left( \bm J_0, \bm \theta_1 \right) = S\left( \bm I_0, \bm \varphi_1  \right) \\
            =\, & S\left( \bm J_0 + \varepsilon \partial_2{\Tc_1 \left( \bm J_0, \bm \varphi_0 \right)}, \bm \varphi_1 \right)    \\
            =\, & S_0 \left( \bm J_0 + \varepsilon \partial_2{\Tc_1 \left( \bm J_0, \bm \varphi_0 \right)} \right) \\
            & + \varepsilon S_1 \left( \bm J_0 + \varepsilon \partial_2{\Tc_1 \left( \bm J_0, \bm \varphi_0 \right)}, \bm \varphi_0 + h \partial_1 S(\bm I_0, \bm \varphi_1) \right) \\
            & + \Oc\left( \varepsilon^2 \right) \\
            =\, & S_0 \left( \bm J_0 \right) + \varepsilon \partial_2 \Tc_1 \left( \bm J_0, \bm \varphi_0 \right) \cdot \nabla S_0 \left( \bm J_0 \right) \\
                & + \varepsilon S_1 \left( \bm J_0, \bm \varphi_0 + h \nabla S(\bm J_0) \right) + \Oc\left( \varepsilon^2 \right) \\
            =\, & S_0 \left( \bm J_0 \right) + \varepsilon \big\{ \underline{ \partial_2 \Tc_1 \left( \bm J_0, \bm \varphi_0 \right) \cdot \nabla S_0 \left( \bm J_0 \right)} \\
            & \underline{+ S_1 \left( \bm J_0, \bm \varphi_0 + h \nabla S(\bm J_0) \right)} \big\}
                + \Oc\left( \varepsilon^2 \right) \\
        \end{split}
    \end{align}
    with $h$ the constant in \cref{eq:prediction}.
    Collecting all $\Oc(\varepsilon^2)$ terms, denoting them by a remainder term $\tilde{R}_2$, and converting all angles in $\tilde{R}_2$ to $\bm \theta_1$, we have
    \begin{align}\label{eq:near-identity}
        \begin{split}
            & \tilde{S} \left( \bm J_0, \bm \theta_1 \right) \\
            = & S_0 \left( \bm J_0 \right) + \varepsilon \big\{ \underline{ \partial_2 \Tc_1 \left( \bm J_0, \bm \varphi_0 \right) \cdot \nabla S_0 \left( \bm J_0 \right)} \\
            & \underline{+ S_1 \left( \bm J_0, \bm \varphi_0 + h \nabla S_0(\bm J_0) \right)} \big\}
                + \varepsilon^2 \tilde{R} \left( \bm J_0, \bm \theta_1, \varepsilon \right). \\
        \end{split}
    \end{align}
    As long as the terms underlined in \cref{eq:near-identity} add up to a function of $\bm J_0$ only, $\tilde{S} \left( \bm J_0, \bm \theta_1 \right)$ won't have angle dependence till the $\Oc\left( \varepsilon^2 \right)$ term.
    This leads to a solvability requirement. More precisely, let $G_1(\bm J_0, \bm \varphi_0)=S_1(\bm J_0, \bm \varphi_0 + h \nabla S_0(\bm J_0))$ and $\overline{G_1}(\bm J_0)$ be its angle average. Then the PDE
    \begin{align}\label{eq:avg_angles}
          \partial_2 \Tc_1 \left( \bm J_0, \bm \varphi_0 \right) \cdot \nabla S_0 \left( \bm J_0 \right) + G_1 \left( \bm J_0, \bm \varphi_0 \right)
          = \overline{G_1}(\bm J_0)
    \end{align}
    has a solution $\Tc_1$, and it makes the underlined terms $\overline{G_1}(\bm J_0)$.
    Therefore, $\Tc_1$ and hence the generating function $\Tc$ can be solved for from \cref{eq:avg_angles}.
    The generating function $\tilde{S}$ in $\Jtheta$ variables takes the form
    \begin{align*}
        \begin{split}
            \tilde{S} \left( \bm J_0, \bm \theta_1 \right)
            & = \tilde{S}_0\left( \bm J_0 \right) + \varepsilon^2 \tilde{R}_2 \left( \bm J_0, \bm \theta_1, \varepsilon \right),
        \end{split} 
    \end{align*}
    with $\tilde{S}_0(\bm J_0) = S_0(\bm J_0) + \varepsilon \overline{G}(\bm J_0)$.
\end{proof}

\begin{remark}
    Note that boundedness of $\tilde{R}_2$ requires some extra condition on $\nabla S_0$ (being Diophantine at some point; see \cref{lemma:neighborhood_bound} for details).
    With bounded $\tilde{R}_2$, the generating function $S\left( \cdot, \cdot \right)$ in $\Iphi$ variables is near integrable of order $\Oc\left( \varepsilon \right)$ 
    while $\tilde{S} \left( \cdot, \cdot \right)$ in $\Jtheta$ variables is near integrable of order $\Oc\left( \varepsilon^2 \right)$.
    Therefore, under the transformation $\Tc$, we get a `better' set of variables $\Jtheta$ instead of $\Iphi$, as the $\Jtheta$ dynamics is closer to being integrable, hence the dynamics of $\Jtheta$ can be estimated for longer time.
\end{remark}

\begin{remark}
    As angles satisfy periodic boundary conditions, $\Tc_1$ and $S_1$ can be expanded in Fourier series
    \begin{align}\label{eq:T1S1_Fourier}
        \left\{
        \begin{aligned}
            \Tc_1 \left( \Jtheta \right) & = \sum_{\bm k \in \Z^d} t_{\bm k} \left( \bm J \right) \cdot e^{i \left( \bm k \cdot \bm \theta \right)}, \\
            S_1 \left( \Jtheta \right) & = \sum_{\bm k \in \Z^d} s_{\bm k} \left( \bm J \right) \cdot e^{i \left( \bm k \cdot \bm \theta \right)}.
        \end{aligned}
        \right.
    \end{align}
    Plugging \cref{eq:T1S1_Fourier} into \cref{eq:T1_const}, we have
    \begin{align*}
        t_{\bm k} \cdot \left(\bm k \cdot \bm \omega \left( \bm J \right) \right) + s_{\bm k}=0.
    \end{align*}
    Noting that if $\bm \omega$ doesn't satisfy Diophantine condition, $\bm k \cdot \bm \omega$ can be small and may even vanish for some $\bm k \in \Z^d$,
    meaning that under some circumstances, the transformation constructed by
    $\Tc \left( \bm J, \bm \varphi \right) = \bm J \cdot \bm \varphi + \varepsilon \Tc_1\left( \bm J, \bm \varphi \right)$ is no longer of near identity ($Id+\Oc(\varepsilon)$) as $\Tc_1$ is not of order $\Oc\left( 1 \right)$ any more.
\end{remark}

\begin{lemma}\label{lemma:transform2}
    Consider a nearly integrable system with generating function $S\left( \bm I_0, \bm \varphi_1 \right) = S_0\left( \bm I_0 \right) + \varepsilon S_1 \left( \bm I_0, \bm \varphi_1 \right)$. Suppose $S_0$ and $S_1$ are analytic and bounded in a complex neighborhood of $\Dc_1 \subseteq \R^d$ and $\Dc = \Dc_1 \times \T^d$ respectively.
    Then, there exists a real analytic canonical transformation $\left( \bm J, \bm \theta \right) \leftrightarrow \left( \bm I, \bm \varphi \right)$
    generated by 
    $\Tc \left( \bm J, \bm \varphi \right) = \bm J \cdot \bm \varphi + \sum_{k=1}^{N-1} \varepsilon^k \cdot \Tc_k\left( \bm J, \bm \varphi \right)$,
    such that the dynamics produced by the original generating function $S$ rewritten in $\Jtheta$ variables corresponds to a transformed generating function $\tilde{S}$ given by
    \begin{align}\label{eq:transform2_eq1}
        \begin{split}
            \tilde{S} \left( \bm J_0, \bm \theta_1 \right)
            & = \tilde{S}_0\left( \bm J_0 \right) + \varepsilon^N \tilde{R}_N \left( \bm J_0, \bm \theta_1, \varepsilon \right), \\
        \end{split}
    \end{align}
    where $\varepsilon^N \tilde{R}_N$ is a high-order perturbation to a new integrable system $\tilde{S}_0(\bm J_0)$.
    Here, $\bm J, \bm \theta \leftrightarrow \bm I, \bm \varphi$ is defined by $\Tc$ through
    \begin{align}\label{eq:canonicalT2}
        \left\{
        \begin{aligned}
            \bm I_i = \partial_2 \Tc \left( \bm J_i, \bm \varphi_i \right), \\
            \bm \theta_i = \partial_1 \Tc \left( \bm J_i, \bm \varphi_i \right), \\
        \end{aligned}
        \quad \forall i=0,1.
        \right.
    \end{align}
\end{lemma}

\begin{proof}
    Apply $\Tc(\bm J, \bm \varphi) = \bm J \cdot \bm \varphi + \varepsilon \Tc_1 + \varepsilon \Tc_2 + \ldots + \varepsilon^{N-1} \Tc_{N-1}$ to $\tilde{S}(\bm J_0, \bm \theta_1) = S(\bm I_0, \bm \varphi_1)$ like in the proof of \cref{lemma:transform1}, Taylor expand, and put all $\Oc(\varepsilon^N)$ terms into $\tilde{R}_N$. Then we have
    \begin{align}\label{eq:order_cmp}
        \begin{split}
            & \tilde{S}(\bm J_0, \bm \theta_1) =\, S_0(\bm J_0) \\
            & + \varepsilon \left(\partial_2 \Tc_1 (\bm J_0, \bm \varphi_0) \cdot \bm \omega(\bm J_0) +  G_1(\bm J_0, \bm \varphi_0) \right) \\
            & + \varepsilon^2  \left( \partial_2 \Tc_2 (\bm J_0, \bm \varphi_0) \cdot \bm \omega(\bm J_0) + G_2(\bm J_0, \bm \varphi_0) \right) \\
            & + \ldots \\
            & + \varepsilon^N \tilde{R}_N \left( \bm J_0, \bm \theta_1, \varepsilon \right), \\
        \end{split}
    \end{align} 
    for some functions $G_1,\cdots,G_{N-1}$ periodic in the angles, and $\bm \omega(\cdot):=\nabla S(\cdot)$. Similar to $G_1$ in the proof of \cref{lemma:transform1}, $G_2, G_3, \cdots$ can be explicitly computed from Taylor expansions, but our proof does not require their specific expressions.
    \cref{lemma:transform1} solved for $\Tc_1$ (periodic in angles) by making the underlined expression independent of the angles.
    Repeating a similar procedure at different orders of $\varepsilon$, $\Tc_i$ (periodic) can be obtained for $i=1,2,\cdots$.
    Specifically, $\Tc_i$ satisfies the PDE
    \begin{align}
        \partial_2 \Tc_i(\bm J_0, \bm \varphi_0) \cdot \bm \omega(\bm J_0) + G_i(\bm J_0, \bm \varphi_0) = \overline{G_i}(\bm J_0)
    \end{align}
    (The existence of the solution to this is proved in \cref{lemma:Ti_bound}).
    In general,
    \begin{align*}
        \begin{split}
            \tilde{S} \left( \bm J_0, \bm \theta_1 \right)
            & = S_0 \left( \bm J_0 \right) + \sum_{k=1}^{N-1} \varepsilon^{k} \overline{G_k} \left( \bm J_0 \right) + \varepsilon^N \tilde{R}_N \left( \bm J_0, \bm \theta_1 \right) \\
            & = \tilde{S}_0\left( \bm J_0 \right) + \varepsilon^N \tilde{R}_N \left( \bm J_0, \bm \theta_1 \right), \\
        \end{split}
    \end{align*}
    with $\tilde{S}_0\left( \bm J_0 \right) = S_0 \left( \bm J_0 \right) + \sum_{k=1}^{N-1} \varepsilon^k \overline{G_k} \left( \bm J_0 \right)$.
\end{proof}

Note that $\tilde{R}_N$ is not necessarily uniformly bounded in the whole data domain in \cref{lemma:transform1,lemma:transform2}, but in most cases, the uniform boundedness can be established (\cref{lemma:neighborhood_bound}) and under that circumstance, we will be able to quantitatively estimate the $\Jtheta$ dynamics (\cref{lemma:Jdynamics_estimates}).



\begin{lemma}\label{lemma:neighborhood_bound}
    Consider a nearly integrable system with generating function $S\left( \bm I_0, \bm \varphi_1 \right) = S_0\left( \bm I_0 \right) + \varepsilon S_1 \left( \bm I_0, \bm \varphi_1 \right)$. Suppose $S_0$ and $S_1$ are analytic and bounded in a complex neighborhood of $\Dc_1$ and $\Dc = \Dc_1 \times \T^d \subseteq \R^d \times \T^d$ respectively.
    There exists a real analytic symplectic change of coordinates of order $\Oc\left( \varepsilon \right)$: $\left( \Iphi \right) \leftrightarrow \left( \Jtheta \right)$ and under this transformation, the generating function in $\Jtheta$ takes the form
    \begin{align*}
        \tilde{S}\left( \bm J_0, \bm \theta_1 \right) = \tilde{S}_0\left( \bm J_0 \right) + \varepsilon^N \tilde{R}_N \left( \bm J_0, \bm \theta_1, \varepsilon \right).
    \end{align*}
    Suppose that $\omega(\bm J^*)$
    satisfies the $(\gamma, \nu)$-Diophantine condition for some $\bm J^* \in \Dc_1$. Then, for any fixed $N \ge 2$, there exist positive constants $\varepsilon_0, c, C, \rho$ such that if $\varepsilon \le \varepsilon_0$, then
    \begin{align*}
        \norm{\tilde{R}_N \left( \cdot, \cdot \right)}_{\infty, \overline{\Bc_{2\delta}\left( \bm J^* \right)} \times \overline{\Bc_\rho\left( \T^d \right)}} \le C
    \end{align*}
    with $\delta = c {\left( N^2 \abs{\log{\varepsilon}} \right)}^{-\nu-1}$.
\end{lemma}

\begin{proof}
    Applying the canonical transformation
    $\Tc$
    constructed in \cref{lemma:transform2}, 
    $\exists \rho', C'>0$ such that
    \begin{align}\label{eq:Cprime_bound}
        \norm{\tilde{R}_N\left( \bm J^*, \cdot \right)}_{\infty, \Bc_{\rho'} (\T^d)} \le C'(N,d,\gamma,\nu)
    \end{align}
    Approximate $S$ with respect to angle variables using Fourier series $\hat{S}_m$ till term $m \propto \abs{\log{\varepsilon}}$ such that the error is of order $\Oc(\varepsilon^N)$ in a complex neighborhood of the torus $\left\{ \bm J = \bm J^*, \bm \varphi \in \T^d \right\}$.
    Since $\abs{\bm k \cdot \bm \omega \left( \bm J^* \right)} \ge \gamma \norm{\bm k}_1^{-\nu}$, $\forall \bm k \in \Z^d$,
    then $\exists$ sufficiently small $c>0$
    such that
    \begin{align}\label{eq:half_gamma_nu}
        \abs{\bm k \cdot \bm \omega\left( \bm J \right)} \ge \frac{1}{2}\gamma\norm{\bm k}_1^{-\nu},
        \, \norm{\bm k}_1 \le N m
    \end{align}
    for all $\bm J \in \Bc_{2\delta}(\bm J^*)$ with $\delta = c {(N^2 \abs{\log \varepsilon})}^{-\nu-1}$.
    As the Fourier coefficients of $\hat{S}_m$ vanishes for $\norm{\bm k}_1 > N m$,
    thus according to condition \cref{eq:half_gamma_nu} and combining $S=\hat{S}_m + \Oc(\varepsilon^N)$, $\exists$ $\rho''>0$ and $C'' > 0$, such that
    \begin{align}\label{eq:Cprimeprime_bound}
        \norm{\tilde{R}_N\left( \bm J, \cdot \right)}_{\infty, \Bc_{\rho''} (\T^d)} \le C''(N,d,\gamma,\nu)
    \end{align}
    for all $\norm{\bm J - \bm J^*} \le 2 \delta$.
    
    In general, $\exists C$ and $\varepsilon$ independent $\rho$ such that
    \begin{align*}
        \norm{\tilde{R}_N\left( \cdot, \cdot \right)}_{\infty, \overline{\Bc_{2 \delta} (\bm J^*)} \times \overline{\Bc_{\rho}(\T^d)}} \le C(N,d,\gamma,\nu).
    \end{align*}
    (for the specific forms of $C', C''$, which are lengthy but obtainable using tools of Fourier series and Cauchy's inequality, see for instance \cite{hairer2006geometric}).
\end{proof}

\begin{lemma}\label{lemma:Jdynamics_estimates}
    Consider a nearly integrable system with generating function $S\left( \bm I_0, \bm \varphi_1 \right) = S_0\left( \bm I_0 \right) + \varepsilon S_1 \left( \bm I_0, \bm \varphi_1 \right)$. Suppose $S_0$ and $S_1$ are analytic and bounded in a complex neighborhood of $\Dc_1$ and $\Dc = \Dc_1 \times \T^d \subseteq \R^d \times \T^d$ respectively.
    Then there exists a real analytic near identity symplectic change of coordinates $\left( \Iphi \right) \mapsto \left( \Jtheta \right)$ of order $\Oc(\varepsilon)$ and under this transformation, the generating function $\tilde{S}$ in $\Jtheta$ variables takes the form 
    \begin{align*}
        \tilde{S}\left( \bm J, \bm \theta \right) = \tilde{S}_0\left( \bm J \right) + \varepsilon^N \tilde{R}_N \left( \bm J, \bm \theta, \varepsilon \right).
    \end{align*}
    where $\tilde{S}_0$ only depends on actions.
    Suppose that $\bm \omega\left( \bm J^* \right)$ satisfies the $\left( \gamma, \nu \right)$-Diophantine condition for some $\bm J^* \in \Dc_1$.
    Then, for any fixed $N \ge 2$, $\exists$ positive constants $\varepsilon_0, c, C, \rho$
    such that if $\varepsilon \le \varepsilon_0$,
    the dynamics of $\Jtheta$ (generated by $\tilde{S}$) with $\norm{\bm J_0 - \bm J^*}_2 \le c \abs{\log{\varepsilon}}^{-\nu-1}$ satisfies
    \begin{align}\label{eq:Jtheta_error_bound}
        \left\{
            \begin{aligned}
                & \norm{\bm J_n - \bm J_0}_2 \le C n h \varepsilon^N, \\
                & \norm{\bm \theta_n - \tilde{\bm \omega} \left( \bm J_0 \right) nh - \bm \theta_0}_2 \\
                & \qquad \le C \left( n^2 h^2 + n h \abs{\log{\varepsilon}}^{\nu+1} \right) \varepsilon^N. \\
            \end{aligned}
        \right.
    \end{align}
    Here $\bm \omega \left( \cdot \right) = \nabla S_0 \left( \cdot \right)$ and $\tilde{\bm \omega} \left( \cdot \right) = \nabla \tilde{S}_0 \left( \cdot \right)$.
\end{lemma}

\begin{proof}
    According to \cref{lemma:neighborhood_bound}, $\exists c>0, \rho>0, C'>0$ such that for $\delta=c {\left( N^2 \abs{\log{\varepsilon}} \right)}^{-\nu-1}$,
    $\bm J \in \overline{\Bc_\delta \left( \bm J^* \right)}$ and $\bm \theta \in \Bc_\rho\left(\T^d\right)$, $\abs{\tilde{R}_N\left( \Jtheta \right)} \le C'$.
    As $\forall \bm J \in \overline{\Bc_\delta \left( \bm J^* \right)}$, $\overline{\Bc_\delta \left( \bm J \right)} \subset \overline{\Bc_{2 \delta} \left( \bm J^* \right)}$,
    $\abs{\tilde{R}_N \left( \tilde{\bm J}, \bm \theta \right)}\le C'$ for all $\tilde{\bm J} \in \overline{\Bc_\delta\left( \bm J \right)}$ and $\bm \theta \in \Bc_\rho \left( \T^d \right)$.
    Using Cauchy's inequality (\cref{thm:Cauchy_inequalities}), we have
    \begin{align}\label{eq:RN1_derivative_estiamtes}
        \norm{\partial_2 \tilde{R}_N}_{\infty, \overline{\Bc_\delta \left( \bm J^* \right)} \times \overline{\Bc_\rho(\T^d)}} \le C'
    \end{align}
    and
    \begin{align}\label{eq:RN2_derivative_estiamtes}
        \norm{\partial_1 \tilde{R}_N}_{\infty, \overline{\Bc_\delta \left( \bm J^* \right)} \times \overline{\Bc_\rho(\T^d)}} \le \frac{C'}{\delta}.
    \end{align}

    Plug \cref{eq:RN1_derivative_estiamtes} in the dynamics of $\Jtheta$
    \begin{align}\label{eq:Jtheta_Stilde}
        \left\{
        \begin{aligned}
            \bm J_i = \bm J_{i+1} + h \partial_2 \tilde{S} \left( \bm J_i, \bm \theta_{i+1} \right), \\
            \bm \theta_{i+1} = \bm \theta_i + h \partial_1 \tilde{S} \left( \bm J_i, \bm \theta_{i+1} \right), \\
        \end{aligned}
    \right.
    \end{align}
    we have
    \begin{align*}
        & \norm{\bm J_{i+1} - \bm J_i}_2 \le C' h \varepsilon^N,\, \forall i \in \N \\
        & \implies \norm{\bm J_n - \bm J_0}_2 \le C'nh\varepsilon^N.
    \end{align*}
    for the $\bm J$ sequence.
    On the other hand, for $\bm \theta$ sequence, plug \cref{eq:RN2_derivative_estiamtes} in  \cref{eq:Jtheta_Stilde}, we have
    \begin{align}\label{eq:theta_i_ip1}
        \norm{\bm \theta_{i+1} - \left(\bm \theta_i + h \bm \tilde{\omega} \left( \bm J_i \right) \right)}_2 & \le \frac{C'}{\delta} h \varepsilon^N.
    \end{align}
    Since $\tilde{\omega}$ is analytic on a bounded domain, $\tilde{\omega}$ is Lipschitz. Thus, changing $\bm J_i$ in \cref{eq:theta_i_ip1} to $\bm J_0$, $\exists C''$ 
    such that 
    \begin{align*}
        \norm{\bm \theta_{i+1} - \left(\bm \theta_i + h \bm \tilde{\omega} \left( \bm J_0 \right) \right)}_2 & \le C'' nh^2 \varepsilon^N + \frac{C'}{\delta} h \varepsilon^N.
    \end{align*}
    Therefore, letting $C = \max(C', C'')$, we have
    \begin{align*}
        & \norm{\bm \theta_{n} - (\bm \theta_0+nh\bm \tilde{\omega}\left( \bm J_0 \right))}_2 \\
        & \le \sum_{i=0}^{n-1} \norm{\bm \theta_{i+1} - \left(\bm \theta_i + h \bm \tilde{\omega} \left( \bm J_0 \right) \right)}_2 \\
        & \le C nh\left(nh + \frac{1}{\delta} \right) \varepsilon^N \\
        & \le C nh \left( nh + \abs{\log{\varepsilon}}^{\nu+1} \right) \varepsilon^N,
    \end{align*}
    and \cref{eq:Jtheta_error_bound} is proved.
\end{proof}

\begin{proof}[Proof of \cref{thm:GFNN_action_angle}]
    Since it is assumed that analytic $S_h$ and $S_h^\theta$ satisfy
        \begin{align*}
            \sum_{i=1,2} \norm{\partial_i S_h^\theta \left( \cdot,\cdot \right) - \partial_i S_h \left( \cdot,\cdot \right)}_{\infty} \le C_1 \varepsilon
        \end{align*}
    on a bounded domain $\Dc$, $S_h^\theta$ is an $\mathcal{O}(\varepsilon)$ perturbation of $S_h$ (note they can also be different by an $\mathcal{O}(1)$ constant, but adding a constant to a generating function does not change its induced dynamics, and we thus assume without loss of generality that there is no such constant difference). Therefore, as $S_h$ is integrable, $S_h^\theta$ can be written as
    $S_h^\theta(\bm I_0, \bm \varphi_1) = S_h(\bm I_0) + \varepsilon S_1(\bm I_0, \bm \varphi_1)$ for some function $S_1$ modeling the (normalized) perturbation, and is thus nearly integrable.
    The latent dynamics, i.e., the exact solution of the integrable $S_h(\bm I_0)$ with initial condition $\Iphit{0}$ is
    \begin{align*}
        \left\{
            \begin{aligned}
                \bm I\left( t \right) & = \bm I_0, \\
                \bm \varphi \left( t \right) & = \left( \bm \varphi_0 + \bm \omega\left( \bm I_0 \right) t \right) \Mod 2 \pi. \\
            \end{aligned}
        \right.
    \end{align*}
 Applying \cref{lemma:Jdynamics_estimates} with $N \ge 3$ (so that the ${(nh)}^2 \varepsilon^N$ term is of order $\Oc(\varepsilon)$ when $nh\varepsilon=\Oc(1)$)
 , there exists a near identity canonical transformation $\left( \Iphi \right) \mapsto \left( \Jtheta \right)$ of order $\varepsilon$ such that
    the solution of $S_h^\theta$ in $\Jtheta$ variable satisfies
    \begin{align}\label{eq:Jtheta_dynamics}
        \left\{
            \begin{aligned}
                & \norm{\bm J_n - \bm J_0}_2 \le C' \varepsilon, \\
                & \norm{\bm \theta_n - \left( \bm \theta_0 + \tilde{\bm \omega} \left( \bm J_0 \right) n h \right)}_2 \le C' \varepsilon n h, \\
            \end{aligned}
        \right.
    \end{align}
    for some constant $C'$ ($\varepsilon$ independent) $\forall n \le h^{-1} \varepsilon^{-1}$ (so that $nh\varepsilon$ is of order $\Oc(1)$)
    with $\tilde{\bm \omega} \left( \cdot \right)$ defined in \cref{lemma:Jdynamics_estimates}.
    Note that the canonical transformation holds for all $\left( \Iphit{i} \right) \leftrightarrow \left( \Jthetat{i} \right)$, we have $\norm{\bm I_i - \bm J_i}_2 \le k \varepsilon$,
    $\norm{\bm \varphi_i - \bm \theta_i}_2 \le k \varepsilon$, $\forall i \in \N$ for some constant $k>0$
    and $\norm{\tilde{\bm \omega} \left( \bm J_0 \right)-\bm \omega\left( \bm I_0 \right)}_2 \le k' \varepsilon$ for some positive constants $k'$.
    Applying triangular inequality, $\exists\, C>0$, such that
    \begin{align*}
        \left\{
            \begin{aligned}
                & \norm{\bm I_n - \bm I_0}_2 \le C \varepsilon, \\
                & \norm{\bm \varphi_n - \left( \bm \varphi_0 + \bm \omega \left( \bm I_0 \right) n h \right)}_2 \le C \varepsilon n h, \\
            \end{aligned}
        \right.
    \end{align*}
    for $n \le h^{-1} \varepsilon^{-1}$.
\end{proof}

\begin{remark}
$h^{-1} \varepsilon^{-1}$ is actually a conservative bound for $n$ and one can extend the bound to be $h^{-1} \varepsilon^{(1-N)/2}$, $\forall N \ge 3$. Since $N$ can be arbitrary, even if $\varepsilon$ cannot be made infinitesimal, as long as it is below a threshold, the time of validity of the error bound in \cref{thm:GFNN_action_angle} can be extended to arbitrarily long.
\end{remark}

\section{Experimental Details}\label{appendix:experimental_details}

Like most neural-network based algorithms for learning dynamics, the full potential of GFNN is achieved in the data rich regime. When preparing training data, we not only prepared in an unbiased way, but also emphasized on fair comparisons so that each of the existing methods is given the same or more training data.

More specifically, for each experiment, the training set contains a number of sequences starting with different initial conditions. When training for predicting Hamiltonian dynamics (continuous), each sequence in the training set is stroboscopically sampled from simulated ground truth, which is obtained using high-order numerical integrator with sufficiently small timestep $\tau \ll h$. For each experiment, the data set with sequences of length $2$ will be denoted as $\Dc_2$, and the data set with sequences of length $5$ will be denoted as $\Dc_5$.
VFNN, HNN, SRNN (seq\_len=2), and GFNN are trained with the same data set $\Dc_2$, while SRNN (seq\_len=5) is trained with $\Dc_5$.
Note the number of flow maps ($\phi$) needed for each sequence in $\Dc_5$ is $4$, while the number of maps for each sequence of $\Dc_2$ is $1$.
Therefore, for fairness, the number of sequences in $\Dc_2$, $n_{train}(\Dc_2)$, is set to be four times $n_{train}(\Dc_5)$ in most examples (exceptions will be explained).

All experiments are performed with PyTorch (CUDA) on a machine with GeForce RTX 3070 graphic card,
AMD Ryzen 7 3700X 8-Core Processor, 16 GB memory and the Linux distribution of openSUSE Leap 15.2.

Codes are provided.

\subsection{2-Body Problem}

The step size of each data sequence in $\Dc_2$ and $\Dc_5$ is $h=0.1$. The ground truth trajectory is simulated using a $4th$ order symplectic integrator with step size $10^{-3}$.
The initial condition of each data sequence is uniformly drawn from the orbits with semi-major axis $a \in (0.8, 1.2)$, eccentricity $e \in (0, 0.05)$.
In terms of the number of samples, $n_{train}(\Dc_2)=100,000$, $n_{train}(\Dc_5)=100,000$.
Note SRNN (seq\_len=5) is provided more training data $n_{train}(\Dc_5)$ than described above, which would be $25,000$ instead, because less training data didn't provide good performance.
The time derivative data of the vector field based methods (VFNN, HNN) are generated using (1st-order) finite difference.

$S_h^\theta$ is represented using multilayer perceptron (MLP), with $5$ layers and $200$, $100$, $50$, $20$ neurons in hidden layers.
The Adam optimizer is utilized with batch size $200$. The model is trained for more than $20$ epochs with initial learning rate $0.01$.
HNN, SRNN, SympNets are trained by their provided codes. HNN, SRNN are trained under default training setups and SympNets is trained using LA-SympNets with 30 layers and 10 sublayers (deeper than their default setups for improved performance).


\subsection{H\'enon-Heiles System}

The step size of each data sequence in $\Dc_2$ and $\Dc_5$ is $h=0.5$. The ground truth trajectory is simulated using a $4th$ order symplectic integrator with step size $10^{-3}$.
The initial condition of each data sequence is drawn from a centered Gaussian perturbation of states along one orbit randomly with variance $0.01^2$.
In terms of the number of samples, $n_{train}(\Dc_2)=100,000$, $n_{train}(\Dc_5)=25,000$.
The data sets for the regular motion experiment and for the chaotic dynamics experiment are generated separably around a trajectory with energy level 
$\frac{1}{12}$ and $\frac{1}{6}$ respectively.
The time derivative data of the vector field based methods (VFNN, HNN) are generated using (1st-order) finite difference.

The MLP that represents $S_h^\theta$ has $5$ layers and $200$, $100$, $50$, $20$ neurons in hidden layers.
The Adam optimizer is utilized with batch size $200$. The model is trained for more than $20$ epochs with initial learning rate $0.01$.
HNN, SRNN are trained by their provided codes under default training setups.

\subsection{PCR3BP}

The step size of each data sequence in $\Dc_2$ and $\Dc_5$ is $h=0.1$. The ground truth trajectory is simulated using RK4 with step size $10^{-3}$.
The initial condition of each data sequence is drawn from a centered Gaussian perturbation of states along one orbit randomly with variance $0.05^2$.
In terms of the number of samples, $n_{train}(\Dc_2)=100,000$, $n_{train}(\Dc_5)=25,000$.
The time derivative data of the vector field based methods (VFNN, HNN) are generated using (1st-order) finite difference.

The MLP that represents $S_h^\theta$ has $5$ layers and $200$, $100$, $50$, $20$ neurons in hidden layers.
The Adam optimizer is utilized with batch size $200$. The model is trained for more than $20$ epochs with initial learning rate $0.01$.
HNN, SRNN are trained by their provided codes under default training setups.

\subsection{Standard Map}

The step size of each data sequence in $\Dc_2$ and $\Dc_5$ is $h=1$. The ground truth map is directly evolved from the discrete-time evolution map \cref{eq:standardMap}.
The initial condition of each data sequence is drawn from a Gaussian perturbation of states along one orbit
randomly with variance $0.5^2$.
In terms of the number of samples in training / testing data, $n_{train}(\Dc_2)=1,000,000$, $n_{train}(\Dc_5)=250,000$.
The data sets of the regular motion experiment and for the chaotic dynamics experiment are generated separably with $K=0.6$ and $K=1.2$ and correspondingly different initial conditions respectively.
The time derivative data of the vector field based methods (VFNN, HNN) are generated using (1st-order) finite difference (with $\Delta t=1$).

The MLP that represents $S_h^\theta$ has $5$ layers and $500$, $500$, $200$, $20$ neurons in hidden layers.
The Adam optimizer is utilized with batch size $1000$. The model is trained for more than $20$ epochs with initial learning rate $0.001$.
HNN, SRNN are trained by their provided codes under default training setups.

\end{document}